\documentclass[11pt]{article}
\usepackage[in]{fullpage}
\usepackage[parfill]{parskip}
\usepackage{amsmath,amsthm,amssymb,bbm}
\usepackage{mathtools}
\usepackage{cases}
\usepackage{dsfont}
\usepackage{microtype}
\usepackage{subfigure}
\usepackage{algorithm,algorithmic}
\usepackage{color}
\usepackage{appendix}

\usepackage{bm}
\usepackage{subfigure}
\usepackage{algorithm,algorithmic}
\usepackage{color}
\usepackage{booktabs}       
\usepackage{appendix}
\usepackage{authblk}

\usepackage{url}
\usepackage[authoryear]{natbib}
\usepackage[colorlinks,citecolor=blue,urlcolor=blue,linkcolor=blue,linktocpage=true]{hyperref}
\usepackage{cleveref} 
\usepackage{enumitem}

\newtheorem{theorem}{Theorem}[section]
\newtheorem{lemma}[theorem]{Lemma}

\newtheorem{corollary}[theorem]{Corollary}
\newtheorem{definition}[theorem]{Definition}

\newtheorem{remark}[theorem]{Remark}
\newtheorem{assumption}[theorem]{Assumption}

\def\E{\mathbb{E}}
\def\P{\mathbb{P}}
\def\Var{\mathrm{Var}}
\def\e{E_{\epsilon,\beta}}
\def\N{\widetilde{N}^k_h}
\def\R{\widetilde{R}^k_h}

\def\p{\widetilde{P}^k_h}
\def\hp{\widehat{P}^k_h}
\def\Re{\text{Regret}}

\begin{document}

\title{Near-Optimal Differentially Private Reinforcement Learning}
\author[1]{Dan Qiao}
\author[1]{Yu-Xiang Wang}
\affil[1]{Department of Computer Science, UC Santa Barbara}
\affil[ ]{\texttt{danqiao@ucsb.edu}, \;
	\texttt{yuxiangw@cs.ucsb.edu}}

\date{}

\maketitle

\begin{abstract}
  Motivated by personalized healthcare and other applications involving sensitive data, we study online exploration in reinforcement learning with differential privacy (DP) constraints. Existing work on this problem established that no-regret learning is possible under joint differential privacy (JDP) and local differential privacy (LDP) but did not provide an algorithm with optimal regret. We close this gap for the JDP case by designing an $\epsilon$-JDP algorithm with a regret of $\widetilde{O}(\sqrt{SAH^2T}+S^2AH^3/\epsilon)$ which matches the information-theoretic lower bound of non-private learning for all choices of $\epsilon> S^{1.5}A^{0.5} H^2/\sqrt{T}$. In the above, $S$, $A$ denote the number of states and actions, $H$ denotes the planning horizon, and $T$ is the number of steps. To the best of our knowledge, this is the first private RL algorithm that achieves \emph{privacy for free} asymptotically as $T\rightarrow \infty$. Our techniques --- which could be of independent interest --- include privately releasing Bernstein-type exploration bonuses and an improved method for releasing visitation statistics. The same techniques also imply a slightly improved regret bound for the LDP case.
\end{abstract}

\newpage
\tableofcontents
\newpage


\section{Introduction}\label{sec:intro}
The wide range application of Reinforcement Learning (RL) based algorithms is becoming paramount in many personalized services, including medical care \citep{raghu2017continuous}, autonomous driving \citep{sallab2017deep} and recommendation systems \citep{afsar2021reinforcement}. In these applications, the learning agent continuously improves its performance by learning from users' private feedback and data. The private data from users, however, usually contain sensitive information. Take recommendation system as an instance, the agent makes recommendation (corresponding to the action in a MDP) according to users' location, age, gender, etc. (corresponding to the state in a MDP), and improves its performance based on users' feedback (corresponding to the reward in a MDP). Unfortunately, it is shown that unless privacy protections are launched, learning agents will implicitly memorize information of individual training data points \citep{carlini2019secret}, even if they are irrelevant for learning \citep{brown2021memorization}, which makes RL agents vulnerable to various privacy attacks.

Differential privacy (DP) \citep{dwork2006calibrating} has become the standard notion of privacy. The output of a differentially private RL algorithm is indistinguishable from its output returned under an alternative universe where any individual user is replaced, thereby preventing the aforementioned privacy risks. However, recent works \citep{shariff2018differentially} show that standard DP is incompatible with sublinear regret bound for contextual bandits. Therefore, a relaxed variant of DP: \emph{Joint Differential Privacy} (JDP) \citep{kearns2014mechanism} is considered. JDP ensures that the output of all other users will not leak much information about any specific user and such notion has been studied extensively in bandits problems\citep{shariff2018differentially,garcelon2022privacy}. In addition, another variant of DP: \emph{Local Differential Privacy} (LDP) \citep{duchi2013local} has drawn more and more attention due to its stronger privacy protection. LDP requires that each user's raw data is privatized before being sent to the agent and LDP has been well studied under bandits \citep{basu2019differential,zheng2020locally}.

Compared to the large body of work on private bandits, existing work that studies private RL is sparser. Under the tabular MDP model, \citet{vietri2020private} first defined JDP and proposed PUCB with regret bound and JDP guarantee. \citet{garcelon2021local} introduced LDP under tabular MDP and designed LDP-OBI with regret bound and LDP guarantee. Recently, \citet{chowdhury2021differentially} provided a general framework for this problem and derived the best-known regret bounds under both JDP and LDP. However, the best known regret bound under $\epsilon$-JDP $\widetilde{O}(\sqrt{SAH^3T}+S^2AH^3/\epsilon)$, although with the additional regret due to JDP being a lower order term, is still sub-optimal by $\sqrt{H}$ compared to the minimax optimal regret $\widetilde{O}(\sqrt{SAH^2T})$\footnote{Under the non-stationary MDP as in this paper, the result in \citet{azar2017minimax} will have additional $\sqrt{H}$ dependence.} \citep{azar2017minimax} without constraints on DP. Therefore, if we run Algorithm 2 of \citet{chowdhury2021differentially}, we not only pay for a constant additional regret $\widetilde{O}(S^2AH^3/\epsilon)$, but also suffer from a multiplicative factor of $\sqrt{H}$. Motivated by this, we want to find out whether it is possible to design an algorithm that has optimal regret bound up to lower order terms while satisfying Joint DP.
 
 \begin{table*}[!t]\label{tab:comparison}
 	\centering
 	\resizebox{\linewidth}{!}{
 		\begin{tabular}{ |c|c|c|c| } 
 			\hline
 			Algorithms & Regret under $\epsilon$-JDP  & Regret under $\epsilon$-LDP & Type of bonus \\
 			\hline 
 			PUCB \citep{vietri2020private} & $\widetilde{O}(\sqrt{S^2AH^3T}+S^2AH^3/\epsilon)^\star$   & NA & Hoeffding \\ 
 			LDP-OBI \citep{garcelon2021local} & NA & $\widetilde{O}(\sqrt{S^2AH^3T}+S^2A\sqrt{H^5T}/\epsilon)^\dagger$  & Hoeffding \\ 
 			Private-UCB-PO \citep{chowdhury2021differentially}& $\widetilde{O}(\sqrt{S^2AH^3T}+S^2AH^3/\epsilon)$  & $\widetilde{O}(\sqrt{S^2AH^3T}+S^2A\sqrt{H^5T}/\epsilon)$ & Hoeffding \\  
 			Private-UCB-VI \citep{chowdhury2021differentially}& $\widetilde{O}(\sqrt{SAH^3T}+S^2AH^3/\epsilon)$  & $\widetilde{O}(\sqrt{SAH^3T}+S^2A\sqrt{H^5T}/\epsilon)$ & Hoeffding \\
 			\textcolor{blue}{DP-UCBVI (Our Algorithm \ref{alg:main})} & \textcolor{blue}{$\widetilde{O}(\sqrt{SAH^2T}+S^2AH^3/\epsilon)$}  & \textcolor{blue}{$\widetilde{O}(\sqrt{SAH^2T}+S^2A\sqrt{H^5T}/\epsilon)$} & \textcolor{blue}{Bernstein} \\
 			\hline
 			Lower bound without DP \citep{jin2018q} & $\Omega(\sqrt{SAH^2T})$ & $\Omega(\sqrt{SAH^2T})$ & NA    \\
 			\hline
 		\end{tabular}
 	}
 	\caption{Comparison of our results (in \textcolor{blue}{blue}) to existing work regarding regret under $\epsilon$-joint differential privacy,  regret under $\epsilon$-local differential privacy and type of bonus. Here $T=KH$ is the number of steps, $S$, $A$, $H$ refer to number of states, number of actions and the planning horizon. Bernstein-type bonus uses the knowledge of estimated variance while Hoeffding-type bonus directly bounds the variance by its uniform upper bound. $\star$: For more discussions about this bound, please refer to \citet{chowdhury2021differentially}. $\dagger$: The original regret bound in \citet{garcelon2021local} is achieved under stationary MDP, and can be translated to the bound stated here by adding $\sqrt{H}$ to the first term. }
 \end{table*}
 
\noindent\textbf{Our contributions.} In this paper, we answer the above question affirmatively by constructing a general algorithm for DP RL: Algorithm \ref{alg:main}. Our contributions are threefold.

\begin{itemize}
\itemsep0em
\item A new upper confidence bound (UCB) based algorithm (DP-UCBVI, Algorithm~\ref{alg:main}) that can be combined with any Privatizer (for JDP or LDP). Under the constraint of $\epsilon$-JDP, DP-UCBVI achieves regret of $\widetilde{O}(\sqrt{SAH^2T}+S^2AH^3/\epsilon)$, which matches the minimax lower bound up to lower order terms.  

\item We propose a novel privatization of visitation numbers that satisfies several nice properties (see Assumption \ref{assump} for details). More importantly, our approach is the first to privatize Bernstein-type bonus, which helps tighten our regret bounds through law of total variance.   

\item Under the $\epsilon$-LDP constraint, DP-UCBVI achieves regret of $\widetilde{O}(\sqrt{SAH^2T}+S^2A\sqrt{H^5T}/\epsilon)$ and improves the best known result \citep{chowdhury2021differentially}. 
\end{itemize}

\subsection{Related work}

Detailed comparisons with existing work on differentially private RL under tabular MDP \citep{vietri2020private,garcelon2021local,chowdhury2021differentially} are given in Table~\ref{tab:comparison}, while we leave more discussions about results on regret minimization to Appendix \ref{sec:erw}. Notably, all existing algorithms privatize Hoeffding-type bonus and suffer from sub-optimal regret bound. In comparison, we privatize Bernstein-type bonus and the non-private part of our regret\footnote{As shown in Table \ref{tab:comparison}, the regret bounds of all DP-RL algorithms contain two parts: one results from running the non-private RL algorithms, while the other is the additional cost due to DP guarantees. Throughout the paper, we use ``non-private part'' to denote the regret from running the non-private RL algorithms.} matches the minimax lower bound in \citet{jin2018q}.

Generally speaking, to achieve DP guarantee under RL, a common approach is to add appropriate noise to existing non-private algorithms, and derive tight regret bounds. We discuss about private algorithms under tabular MDP below and leave more discussions about algorithms under other settings to Appendix \ref{sec:erw}. Under the constraint of JDP, \citet{vietri2020private} designed PUCB by privatizing UBEV \citep{dann2017unifying}. Private-UCB-VI \citep{chowdhury2021differentially} resulted from UCBVI (with bonus 1) \citep{azar2017minimax}. Under the constraint of LDP, \citet{garcelon2021local} designed LDP-OBI based on UCRL2 \citep{jaksch2010near}. However, all these works privatized Hoeffding-type bonus, which is easier to handle, but will lead to sub-optimal regret bound. In contrast, we directly build upon the non-private algorithm with minimax optimal regret bound: UCBVI with bonus 2 \citep{azar2017minimax}, where the privatization of Bernstein-type bonus requires more advanced techniques. 

A concurrent work \citep{qiao2022offline} focused on the offline RL setting and derived a private version of APVI \citep{yin2021towards}. Their algorithm achieved tight sub-optimality bound of the output policy through privatization of Bernstein-type pessimism\footnote{Pessimism is the counterpart of bonus under offline RL, which aims to discourage the choice of $(s,a)$ pairs with large uncertainty.}. However, their analysis relied on the assumption that the visitation numbers of all (state,action) pairs are larger than some threshold. We overcome the requirement of such assumption via an improved privatization of visitation numbers. More importantly, offline RL can be viewed as one step of online RL, therefore privatization of Bernstein type bonus is more technically demanding. Finally, our approach actually realizes the future direction stated in the conclusion of \citet{qiao2022offline}.

\subsection{A remark on technical novelty.}
The general idea behind the previous differentially private algorithms under tabular MDP \citep{vietri2020private,garcelon2021local,chowdhury2021differentially} is to add noise to accumulative visitation numbers, and construct a private bonus based on privatized visitation numbers. Since Hoeffding-type bonus $b^k_h(s,a)$ only uses the information of visitation numbers (\emph{e.g.}, in \citet{azar2017minimax}, $b^k_h(s,a)=\widetilde{O}(H\cdot\sqrt{1/N^k_h(s,a)})$), the construction of private bonus is straightforward. We can simply replace original counts $N^k_h(s,a)$ with private counts $\widetilde{N}^k_h(s,a)$ and add an additional term to account for the difference between these two bonuses. Next, combining the construction of private bonuses with the uniform upper bound of $|\widetilde{N}^k_h(s,a)-N^k_h(s,a)|$, we can upper bound the private bonus by its non-private counterpart plus some additional lower order term. Therefore the proof schedule of the original non-private algorithms also applies to their private counterparts.

Unfortunately, although the idea to privatize UCBVI with bonus 2 (Bernstein-type) \citep{azar2017minimax} is straightforward, the generalization of the previous approaches is technically non-trivial. Since the bonus 2 in \citet{azar2017minimax} includes the term $\Var_{\widehat{P}^k_h(\cdot|s,a)}V^k_{h+1}(\cdot)$, the first technical challenge is to replace the empirical transition kernel $\widehat{P}^k_h(s,a)$ with a private estimate. However, the private transition kernel estimates constructed in previous works are not valid probability distributions. In this paper, for both JDP and LDP, we propose a novel privatization of visitation numbers such that the private transition kernel estimates are valid probability distributions and meanwhile, the upper bound on $|\widetilde{N}^k_h(s,a)-N^k_h(s,a)|$ is the same scale compared to previous approaches. With the private transition kernel estimates $\widetilde{P}^k_h$, we can replace $\Var_{\widehat{P}^k_h(\cdot|s,a)}V^k_{h+1}(\cdot)$ with $\Var_{\widetilde{P}^k_h(\cdot|s,a)}\widetilde{V}^k_{h+1}(\cdot)$ where $\widetilde{V}^k_h(\cdot)$ is the value function calculated from value iteration with private estimates. Then the second challenge is to bound the difference between these two variances and retain the optimism. We overcome the second challenge via concentration inequalities. Briefly speaking, we add an additional term (using private statistics) to compensate for the difference of these two bonuses and recovered the proof of optimism. With all these techniques, we derive our regret bound using techniques like error decomposition and error propagation originated from \citet{azar2017minimax}.


\section{Notations and Problem Setup}\label{sec:setup}
Throughout the paper, for $N\in\mathbb{Z}^{+}$, $[N]=\{1,2,\cdots,N\}$. For any set $W$, $\Delta(W)$ denotes the set of all probability distributions over $W$. Besides, we use standard notations such as $O$ and $\Omega$ to suppress constants while $\widetilde{O}$ and $\widetilde{\Omega}$ absorb logarithmic factors. 

Below we present the definition of episodic Markov Decision Processes and introduce differential privacy in reinforcement learning.

\subsection{Markov decision processes and regret} 
We consider finite-horizon episodic \emph{Markov Decision Processes} (MDP) with non-stationary transitions, denoted by a tuple $\mathcal{M}=(\mathcal{S}, \mathcal{A}, H, \{P_h\}_{h=1}^{H}, \{r_h\}_{h=1}^{H}, d_1)$ \citep{sutton1998reinforcement}, where $\mathcal{S}$ is state space with $|\mathcal{S}|=S$, $\mathcal{A}$ is action space with $|\mathcal{A}|=A$ and $H$ is the horizon. The non-stationary transition kernel has the form $P_h:\mathcal{S}\times\mathcal{A}\times\mathcal{S} \mapsto [0, 1]$  with $P_{h}(s^{\prime}|s,a)$ representing the probability of transition from state $s$, action $a$ to next state $s^\prime$ at time step $h$. In addition, $r_h(s,a)\in\Delta([0,1])$ denotes the corresponding distribution of reward, we overload the notation so that $r$ also denotes the expected (immediate) reward function. Besides, $d_1$ is the initial state distribution. A policy can be seen as a series of mapping $\pi=(\pi_1,\cdots,\pi_H)$, where each $\pi_h$ maps each state $s \in \mathcal{S}$ to a probability distribution over actions, \emph{i.e.} $\pi_h: \mathcal{S}\rightarrow \Delta(\mathcal{A})$, $\forall\, h\in[H]$. A random trajectory $ (s_1, a_1, r_1, \cdots, s_H,a_H,r_H,s_{H+1})$ is generated by the following rule: $s_1\sim d_1$, $a_h \sim \pi_h(\cdot|s_h), r_h \sim r_h(s_h, a_h), s_{h+1} \sim P_h (\cdot|s_h, a_h), \forall\, h \in [H]$. 

Given a policy $\pi$ and any $h\in[H]$, the value function $V^\pi_h(\cdot)$ and Q-value function $Q^\pi_h(\cdot,\cdot)$ are defined as:
$
V^\pi_h(s)=\mathbb{E}_\pi[\sum_{t=h}^H r_{t}|s_h=s] ,
Q^\pi_h(s,a)=\mathbb{E}_\pi[\sum_{t=h}^H  r_{t}|s_h,a_h=s,a],\;\forall\, s,a\in\mathcal{S}\times\mathcal{A}.
$ The optimal policy $\pi^\star$ maximizes $V_h^\pi(s)$ for all $s,h\in\mathcal{S}\times[H]$ simultaneously and we denote the value function and Q-value function with respect to $\pi^\star$ by $V^\star_h(\cdot)$ and $Q^\star_h(\cdot,\cdot)$. Then Bellman (optimality) equation follows $\forall\, h\in[H]$:
\begin{align*}
&Q^\pi_h(s,a)=r_{h}(s,a)+P_{h}(\cdot|s,a)V^\pi_{h+1},\;\;V^\pi_h=\mathbb{E}_{a\sim\pi_h}[Q^\pi_h],\\
&Q^\star_h(s,a)=r_{h}(s,a)+P_{h}(\cdot|s,a)V^\star_{h+1},\; V^\star_h=\max_a Q^\star_h(\cdot,a).
\end{align*}

We measure the performance of online reinforcement learning algorithms by the regret. The regret of an algorithm is defined as
$$\text{Regret}(K) := \sum_{k=1}^{K}[V_{1}^\star(s_{1}^k)-V_{1}^{\pi_{k}}(s_{1}^k)],$$
where $s_1^k$ is the initial state and $\pi_{k}$ is the policy deployed at episode $k$. Let $K$ be the number of episodes that the agent plan to play and total number of steps is $T := KH$.

\subsection{Differential privacy under episodic RL}
Under the episodic RL setting, each trajectory represents one specific user. We first consider the following RL protocol: during the $h$-th step of the $k$-th episode, user $u_k$ sends her state $s_h^k$ to agent $\mathcal{M}$, $\mathcal{M}$ sends back an action $a_h^k$, and finally $u_k$ sends her reward $r_h^k$ to $\mathcal{M}$. Formally, we denote a sequence of $K$ users who participate in the above RL protocol by $\mathcal{U}=(u_1,\cdots,u_K)$. Following the definition in \citet{vietri2020private}, each user can be seen as a tree of depth $H$ encoding the state and reward responses they would reply to all $A^H$ possible sequences of actions from the agent. We let $\mathcal{M}(\mathcal{U})=(a_1^1,\cdots,a_H^K)$ denote the whole sequence of actions chosen by agent $\mathcal{M}$. An ideal privacy preserving agent would guarantee that $\mathcal{M}(\mathcal{U})$ and all users but $u_k$ together will not reveal much information about user $u_k$. We formalize such privacy preservation through adaptation of differential privacy \citep{dwork2006calibrating}.

\begin{definition}[Differential Privacy (DP)]\label{def:dp}
For any $\epsilon>0$ and $\delta\in[0,1]$, a mechanism $\mathcal{M}:\mathcal{U}\rightarrow \mathcal{A}^{KH}$ is $(\epsilon,\delta)$-differentially private if for any possible user sequences $\mathcal{U}$ and $\mathcal{U}^{\prime}$ differing on a single user and any subset $E$ of $\mathcal{A}^{KH}$,
\begin{equation*}
    \P[\mathcal{M}(\mathcal{U})\in E]\leq e^\epsilon \P[\mathcal{M}(\mathcal{U}^\prime) \in E]+\delta.
\end{equation*}
If $\delta=0$, we say that $\mathcal{M}$ is $\epsilon$-differentially private ($\epsilon$-DP).
\end{definition}

However, although recommendation to other users will not affect the privacy of user $u_k$ significantly, it is impractical to privately recommend actions to user $u_k$ while protecting the information of her state and reward. Therefore, the notion of DP is relaxed to \emph{Joint Differential Privacy} (JDP) \citep{kearns2014mechanism}, which requires that for all user $u_k$, the recommendation to all users but $u_k$ will not reveal much information about $u_k$. JDP is weaker than DP, while JDP can still provide strong privacy protection since it protects a specific user from any possible collusion of all other users against her. Formally, the definition of JDP is shown below.

\begin{definition}[Joint Differential Privacy (JDP)]\label{def:jdp}
For any $\epsilon>0$, a mechanism $\mathcal{M}:\mathcal{U}\rightarrow \mathcal{A}^{KH}$ is $\epsilon$-joint differentially private if for any $k\in[K]$, any user sequences $\mathcal{U}$, $\mathcal{U}^\prime$ differing on the $k$-th user and any subset $E$ of $\mathcal{A}^{(K-1)H}$,
\begin{equation*}
    \P[\mathcal{M}_{-k}(\mathcal{U})\in E]\leq e^\epsilon \P[\mathcal{M}_{-k}(\mathcal{U}^\prime) \in E],
\end{equation*}
where $\mathcal{M}_{-k}(\mathcal{U})\in E$ means the sequence of actions recommended to all users but $u_k$ belongs to set $E$.
\end{definition}

JDP ensures that even if an adversary can observe the recommended actions to all users but $u_k$, it is impossible to identify the trajectory from $u_k$ accurately. JDP is first defined and analyzed under RL by \citet{vietri2020private}.

Although JDP provides strong privacy protection, the agent can still observe the raw trajectories from users. Under some circumstances, however, the users are not even willing to share their original data with the agent. This motivates a stronger notion of privacy which is called \emph{Local Differential Privacy} (LDP) \citep{duchi2013local}. Since under LDP, the agent is not allowed to directly observe the state of users, we consider the following RL protocol for LDP: during the $k$-th episode, the agent $\mathcal{M}$ sends policy $\pi_k$ to user $u_k$, after deploying $\pi_k$ and getting trajectory $X_k$, user $u_k$ privatizes her trajectory to $X_k^\prime$ and finally sends it to $\mathcal{M}$. We denote the privacy mechanism on user's side by $\widetilde{\mathcal{M}}$ and define local differential privacy formally below.

\begin{definition}[Local Differential Privacy (LDP)]
For any $\epsilon>0$, a mechanism $\widetilde{\mathcal{M}}$ is $\epsilon$-local differentially private if for any possible trajectories $X,X^\prime$ and any possible set\\ $E\subseteq\{\widetilde{\mathcal{M}}(X)|X\;\text{is any possible trajectory}\}$,
\begin{equation*}
    \P[\widetilde{\mathcal{M}}(X)\in E]\leq e^\epsilon \P[\widetilde{\mathcal{M}}(X^\prime) \in E].
\end{equation*}
\end{definition}

Local DP ensures that even if an adversary observes the whole reply from user $u_k$, it is still statistically hard to identify her trajectory. LDP is first defined and analyzed under RL by \citet{garcelon2021local}.


\section{Algorithm}\label{sec:alg}

\begin{algorithm*}[tbh]
	\caption{DP-UCBVI}\label{alg:main}
	\begin{algorithmic}[1]
		\STATE \textbf{Input}: Number of episodes $K$, privacy budget $\epsilon$, failure probability $\beta$ and a Privatizer (can be either Central or Local).
		\STATE \textbf{Initialize}: Private counts $\widetilde{R}^1_h(s,a)=\widetilde{N}^1_h(s,a)=\widetilde{N}^1_h(s,a,s^\prime)=0$ for all $(h,s,a,s^\prime)\in[H]\times\mathcal{S}\times\mathcal{A}\times\mathcal{S}$. Set up the confidence bound $E_{\epsilon,\beta}$ w.r.t the Privatizer. $\iota=\log(30HSAT/\beta)$. 
		\FOR{$k=1,2,\cdots,K$}  
		\STATE $\widetilde{V}^k_{H+1}(\cdot)=0$.
		\FOR{$h=H,H-1,\cdots,1$}
		\STATE Compute $\widetilde{P}^k_h(s^\prime|s,a)$ and $\widetilde{r}^k_h(s,a)$ as in \eqref{equ:privateest}.
		\STATE Calculate private bonus $b^k_h(s,a)=2\sqrt{\frac{\Var_{s^\prime\sim\widetilde{P}_h^k(\cdot|s,a)}\widetilde{V}^k_{h+1}(\cdot)\cdot\iota}{\widetilde{N}^k_h(s,a)}}+\sqrt{\frac{2\iota}{\widetilde{N}^k_h(s,a)}}+\frac{20HSE_{\epsilon,\beta}\cdot\iota}{\widetilde{N}^k_h(s,a)}+4\sqrt{\iota}\cdot\sqrt{\frac{\sum_{s^\prime}\widetilde{P}^k_h(s^\prime|s,a)\min\left\{\frac{1000^2H^3SA\iota^2}{\widetilde{N}^k_{h+1}(s^\prime)}+\frac{1000^2H^4S^4A^2\e^2\iota^4}{\widetilde{N}^k_{h+1}(s^\prime)^2}+\frac{1000^2H^6S^4A^2\iota^4}{\widetilde{N}^k_{h+1}(s^\prime)^2},H^2\right\}}{\widetilde{N}^k_h(s,a)}}$.
		\FOR{$(s,a)\in\mathcal{S}\times\mathcal{A}$}
		\STATE $\widetilde{Q}^k_h(s,a)=\min\{\widetilde{Q}^{k-1}_h(s,a),H,\widetilde{r}^k_h(s,a)+\sum_{s^\prime}\widetilde{P}_h^k(s^\prime|s,a)\cdot\widetilde{V}^k_{h+1}(s^\prime)+b^k_h(s,a)\}$.
		\ENDFOR
		\FOR{$s\in\mathcal{S}$}
		\STATE $\widetilde{V}^k_h(s)=\max_{a\in\mathcal{A}}\widetilde{Q}^k_h(s,a)$.
		\STATE $\pi_h^k(s)=\arg\max_{a\in\mathcal{A}} \widetilde{Q}^k_h(s,a)$ with ties broken arbitrarily.
		\ENDFOR
		\ENDFOR
		\STATE Deploy policy $\pi_k=(\pi^k_1,\cdots,\pi^k_H)$ and get trajectory $(s_1^k,a_1^k,r_1^k,\cdots,s_{H+1}^k)$.
		\STATE Update the private counts to $\widetilde{R}^{k+1},\widetilde{N}^{k+1}$ via Privatizer.
		\ENDFOR
	\end{algorithmic}
\end{algorithm*}

In this section, we propose DP-UCBVI (Algorithm \ref{alg:main}) that takes Privatizer as input, where the Privatizer can be either Central (for JDP) or Local (for LDP). We provide regret analysis for all privatizers satisfying the following Assumption \ref{assump}, which naturally implies regret bounds under both Joint DP and Local DP.

We begin with the following definition of counts. Let $N^k_h(s,a)=\sum_{i=1}^{k-1}\mathds{1}(s_h^i,a_h^i=s,a)$ denote the visitation number of $(s,a)$ at step $h$ before the $k$-th episode. Similarly, $N^k_h(s,a,s^\prime)=\sum_{i=1}^{k-1}\mathds{1}(s_h^i,a_h^i,s_{h+1}^i=s,a,s^\prime)$ and $R^k_h(s,a)=\sum_{i=1}^{k-1}\mathds{1}(s_h^i,a_h^i=s,a)\cdot r_h^i$ denote the visitation number of $(h,s,a,s^\prime)$ and accumulative reward at $(h,s,a)$ before the $k$-th episode. In non-private RL, such counts are sufficient for estimating transition kernel $P_h$, reward function $r_h$ and deciding the exploration policy, as in \citet{azar2017minimax}. However, these counts are derived from the raw trajectories of the users, which could contain sensitive information. Therefore, under the constraint of privacy, we can only use these counts in a privacy-preserving way, \emph{i.e.} we use the private counts $\widetilde{N}^k_h(s,a),\widetilde{N}^k_h(s,a,s^\prime),\widetilde{R}^k_h(s,a)$ returned by Privatizer. We make the Assumption \ref{assump} below, which says that with high probability, the private counts are close to real ones, such assumption will be justified by our Privatizers in Section \ref{sec:privatizer}.

\begin{assumption}[Private counts]\label{assump}
We assume that for any privacy budget $\epsilon>0$ and failure probability $\beta\in[0,1]$, the private counts returned by Privatizer satisfies that for some $E_{\epsilon,\beta}>0$, with probability at least $1-\beta/3$, uniformly over all $(h,s,a,s^\prime,k)\in[H]\times\mathcal{S}\times\mathcal{A}\times\mathcal{S}\times[K]$: \\
(1) $|\widetilde{N}^k_h(s,a,s^\prime)-N^k_h(s,a,s^\prime)|\leq E_{\epsilon,\beta}$, $|\widetilde{N}^k_h(s,a)-N^k_h(s,a)|\leq E_{\epsilon,\beta}$ and $|\widetilde{R}^k_h(s,a)-R^k_h(s,a)|\leq E_{\epsilon,\beta}$.\\
(2) $\widetilde{N}^k_h(s,a)=\sum_{s^\prime\in\mathcal{S}}\widetilde{N}^k_h(s,a,s^\prime)\geq N^k_h(s,a)$. $\widetilde{N}^k_h(s,a,s^\prime)>0$. Also, we let $\widetilde{N}^k_h(s)=\sum_{a\in\mathcal{A}}\widetilde{N}^k_h(s,a)$.
\end{assumption}

Under Assumption \ref{assump}, for all $(h,s,a,s^\prime,k)$, we define the private estimations of transition kernel and reward function.

\begin{equation}\label{equ:privateest}
    \begin{split}
        \widetilde{P}^k_h(s^\prime|s,a)=\frac{\widetilde{N}^k_h(s,a,s^\prime)}{\widetilde{N}^k_h(s,a)},\\
        \widetilde{r}^k_h(s,a)=\left(\frac{\widetilde{R}^k_h(s,a)}{\widetilde{N}^k_h(s,a)}\right)_{[0,1]}.
    \end{split}
\end{equation}

\begin{remark}
Different from the private empirical transition kernels in \citet{vietri2020private,garcelon2021local,chowdhury2021differentially}, Assumption \ref{assump} implies that our estimated transition kernel $\widetilde{P}^k_h(\cdot|s,a)$ is a valid probability distribution, this property results from our construction of Privatizer. We truncate the empirical reward function so that it stays in $[0,1]$ while still preserving privacy.
\end{remark}

\noindent\textbf{Algorithmic design.} Similar to non-private algorithms \citep{azar2017minimax}, DP-UCBVI (Algorithm \ref{alg:main}) follows the procedure of optimistic value iteration. More specifically, in episode $k$, we do value iteration based on private estimations $\widetilde{P}^k_h$, $\widetilde{r}^k_h$ and private bonus term $b^k_h$ to derive private Q-value functions $\widetilde{Q}^k_h$. Next, the greedy policy $\pi_k$ w.r.t $\widetilde{Q}^k_h$ is chosen and we collect one trajectory by running $\pi_k$. Finally, the Privatizer translates the non-private counts to private ones for the next episode. We highlight that, different from all previous works regarding private RL, our bonus is variance-dependent. According to Law of total variance, variance-dependent bonus can effectively save a factor of $\sqrt{H}$ in regret bound. Intuitively, the first term of $b^k_h$ aims to approximate the variance w.r.t to $V^\star_h$, the last term accounts for the difference between these two variances and the third term is the additional bonus due to differential privacy.


\section{Main results}\label{sec:result}
In this section, we present our main results that formalize the algorithmic ideas discussed in previous sections. We first state a general result based on Assumption \ref{assump}, which can be combined with any Privatizers. The proof of Theorem \ref{thm:main} is sketched in Section \ref{sec:sketch} with details in Appendix \ref{sec:appp}.

\begin{theorem}\label{thm:main}
For any privacy budget $\epsilon>0$, failure probability $0<\beta<1$ and any Privatizer that satisfies Assumption \ref{assump}, with probability at least $1-\beta$, the regret of DP-UCBVI (Algorithm \ref{alg:main}) is
\begin{equation}
    \text{Regret}(K) \leq \widetilde{O}(\sqrt{SAH^2T}+S^2AH^2\e),
\end{equation}
where $K$ is the number of episodes and $T=HK$.
\end{theorem}

Under Assumption \ref{assump}, the best known regret bound is $\widetilde{O}(\sqrt{SAH^3T}+S^2AH^2\e)$ (Theorem 4.2 of \citet{chowdhury2021differentially}). As a comparison, in our regret bound, the term parameterized by privacy loss $\epsilon$ remains the same while the leading term is improved by a factor of $\sqrt{H}$ into $\widetilde{O}(\sqrt{SAH^2T})$. 
More importantly, when $T$ is sufficiently large, our result nearly matches the lower bound in \citet{jin2018q}, hence is information-theoretically optimal up to a logarithmic factor. 


\section{Choice of Privatizers}\label{sec:privatizer}
In this section, we design Privatizers that satisfy Assumption \ref{assump} and different DP constraints (JDP or LDP). All the proofs in this section are deferred to Appendix \ref{sec:proofprivate}.

\subsection{Central Privatizer for Joint DP}
The Central Privatizer protects the information of all single users by privatizing all the counter streams $N^k_h(s,a)$, $N^k_h(s,a,s^\prime)$ and $R^k_h(s,a)$ using the Binary Mechanism \citep{chan2011private}, which focused on privately releasing data stream \citep{zhao2022differentially}. More specifically, for each $(h,s,a)$,  $\{N^k_h(s,a)=\sum_{i=1}^{k-1}\mathds{1}(s_h^i,a_h^i=s,a)\}_{k\in[K]}$ is the partial sums of data stream $\{\mathds{1}(s_h^i,a_h^i=s,a)\}_{i\in[K]}$. Binary Mechanism works as below: for each episode $k$, after observing $\mathds{1}(s_h^{k-1},a_h^{k-1}=s,a)$, the mechanism outputs private version of $\sum_{i=1}^{k-1}\mathds{1}(s_h^i,a_h^i=s,a)$ while ensuring Differential Privacy.\footnote{For more details about Binary Mechanism, please refer to \citet{chan2011private} or \citet{kairouz2021practical}.} Given privacy budget $\epsilon>0$, we construct the Central Privatizer as below:

\begin{enumerate}[leftmargin=0cm,itemindent=.5cm,labelwidth=\itemindent,labelsep=0cm,align=left]
    \item[(1)] For all $(h,s,a,s^\prime)$, we privatize $\{N^k_h(s,a)\}_{k\in[K]}$ and $\{N^k_h(s,a,s^\prime)\}_{k\in[K]}$ (which is summation of bounded streams) by applying Binary Mechanism (Algorithm 2 in \citet{chan2011private}) with $\epsilon^\prime=\frac{\epsilon}{3H\log K}$. We denote the output of Binary Mechanism by $\widehat{N}^k_h$.
    \item[(2)] The private counts $\widetilde{N}^k_h$ are solved through the procedure in Section \ref{sec:post} with\\ $\e=O(\frac{H}{\epsilon}\log(HSAT/\beta)^2)$.
    \item[(3)] For the counters of accumulative reward, for all $(h,s,a)$, we apply the same Binary Mechanism with $\epsilon^\prime=\frac{\epsilon}{3H\log K}$ to privatize $R^k_h(s,a)$ and get $\widetilde{R}^k_h(s,a)$.
\end{enumerate}

We sum up the properties of Central Privatizer below.

\begin{lemma}\label{lem:central}
For any $\epsilon >0$ and $0<\beta<1$, the Central Privatizer satisfies $\epsilon$-JDP and Assumption \ref{assump} with $\e=\widetilde{O}(\frac{H}{\epsilon})$.
\end{lemma}

Therefore, combining Lemma \ref{lem:central} and Theorem \ref{thm:main}, the following regret bound holds.

\begin{theorem}[Regret under JDP]\label{thm:central}
For any $\epsilon>0$ and $0<\beta<1$, running DP-UCBVI (Algorithm \ref{alg:main}) with Central Privatizer as input, with probability $1-\beta$, it holds that:
\begin{equation}
    \text{Regret}(K) \leq \widetilde{O}(\sqrt{SAH^2T}+S^2AH^3/\epsilon).
\end{equation}
\end{theorem}

Under the most prevalent regime where the privacy budget $\epsilon$ is a constant, the additional regret bound due to JDP is a lower order term. The main term of Theorem \ref{thm:central} improves the best known result $\widetilde{O}(\sqrt{SAH^3T})$ (Corollary 5.2 of \citet{chowdhury2021differentially}) by $\sqrt{H}$ and matches the minimax lower bound without constrains on DP \citep{jin2018q}. 

\subsubsection{A post-processing step}\label{sec:post}
During the $k$-th episode, given the noisy counts $\widehat{N}^k_h(s,a)$ and $\widehat{N}^k_h(s,a,s^\prime)$ (for all $(h,s,a,s^\prime)\in[H]\times\mathcal{S}\times\mathcal{A}\times\mathcal{S}$), we construct the following private counts that satisfy Assumption \ref{assump}. The choice of $\widetilde{N}^k_h$ follows: for all $(h,s,a)$

\begin{equation}\label{eqn:final_choice}
\begin{aligned}
&\{\widetilde{N}^k_h(s,a,s^{\prime})\}_{s^{\prime}\in\mathcal{S}}=\underset{\{x_{s^{\prime}}\}_{s^{\prime}\in\mathcal{S}}}{\operatorname*{argmin}}\; \left(\max_{s^{\prime}\in\mathcal{S}} \left|x_{s^{\prime}}-\widehat{N}^k_h(s,a,s^\prime)\right|\right) \\ &\text{such that}\,\, \left|\sum_{s^{\prime}\in\mathcal{S}}x_{s^{\prime}}-\widehat{N}^k_h(s,a)\right|\leq \frac{\e}{4}\,\,\text{and}\,\,x_{s^{\prime}}\geq 0, \forall\, s^{\prime}. \\
&\widetilde{N}^k_h(s,a)=\sum_{s^{\prime}\in\mathcal{S}}\widetilde{N}^k_h(s,a,s^\prime).
\end{aligned}
\end{equation}

Finally, for all $(h,s,a)$, we add the following terms such that with high probability, $\widetilde{N}^k_h(s,a)$ will never underestimate. 
\begin{equation}\label{eqn:addsth}
\begin{split}
\widetilde{N}^k_h(s,a,s^{\prime})=\widetilde{N}^k_h(s,a,s^{\prime})+\frac{\e}{2S},\\\widetilde{N}^k_h(s,a)=\widetilde{N}^k_h(s,a)+\frac{\e}{2}.
\end{split}
\end{equation}

\begin{remark}
The optimization problem \eqref{eqn:final_choice} can be reformulated as:
\begin{equation}\label{eqn:reformulate}
\begin{split}
&\min\,\,t,\,\,\text{s.t.}\, |x_{s^{\prime}}-\widehat{N}^k_h(s,a,s^\prime)|\leq t,\;x_{s^{\prime}}\geq 0,\,\,\forall\, s^\prime\in\mathcal{S},\\& \left|\sum_{s^{\prime}\in\mathcal{S}}x_{s^{\prime}}-\widehat{N}^k_h(s,a)\right|\leq \frac{\e}{4}.
\end{split}
\end{equation}
Note that \eqref{eqn:reformulate} is a \textsf{Linear Programming} problem with $O(S)$ variables and $O(S)$ linear constraints. This  can be solved efficiently by the simplex method \citep{ficken2015simplex} or other provably efficient algorithms \citep{nemhauser1988polynomial}. Therefore, since during the whole process, we only solve $HSAK$ such Linear Programming problems, our Algorithm \ref{alg:main} is computationally efficient.
\end{remark}

The properties of private counts $\widetilde{N}^k_h$ is summarized below.

\begin{lemma}\label{lem:middle}
Suppose $\widehat{N}^k_h$ satisfies that with probability $1-\frac{\beta}{3}$, uniformly over all $(h,s,a,s^\prime,k)$, it holds that
$$|\widehat{N}^k_h(s,a,s^\prime)-N^k_h(s,a,s^\prime)|\leq \frac{\e}{4},$$
$$|\widehat{N}^k_h(s,a)-N^k_h(s,a)|\leq \frac{\e}{4},$$
then the $\widetilde{N}^k_h$ derived from \eqref{eqn:final_choice} and \eqref{eqn:addsth} satisfies Assumption \ref{assump}.
\end{lemma}

\begin{remark}
Compared to the concurrent work \citep{qiao2022offline}, our private counts $\widetilde{N}^k_h(s,a)$ have additional guarantee of never underestimating the true values, which is a desirable property for analysis in Appendix \ref{sec:app}. In comparison, the analysis in \citet{qiao2022offline} heavily relies on the assumption that the visitation number is larger than some threshold such that the scale of noise is ignorable.
\end{remark}

\subsection{Local Privatizer for Local DP}
For each episode $k$, the Local Privatizer privatizes each single trajectory by perturbing the statistics calculated from that trajectory. For visitation of (state,action) pairs, the original visitation number $\{\sigma^k_h(s,a)=\mathds{1}(s_h^k,a_h^k=s,a)\}_{(h,s,a)}$ has $\ell_1$ sensitivity $H$. Therefore, the perturbed version of the counts above $\widetilde{\sigma}^k_h(s,a)=\sigma^k_h(s,a)+\text{Lap}(\frac{3H}{\epsilon})$ satisfies $\frac{\epsilon}{3}$-LDP. In addition, similar perturbations to $\{\mathds{1}(s^k_h,a^k_h,s^k_{h+1}=s,a,s^\prime)\}_{(h,s,a,s^\prime)}$ and $\{\mathds{1}(s_h^k,a_h^k=s,a)\cdot r^k_h\}_{(h,s,a)}$ will lead to the same result. As a result, we construct Local Privatizer as below: 
\begin{enumerate}[leftmargin=0cm,itemindent=.5cm,labelwidth=\itemindent,labelsep=0cm,align=left]
    \item[(1)] For all $(k,h,s,a,s^\prime)$, we perturb $\sigma^k_h(s,a)=\mathds{1}(s_h^k,a_h^k=s,a)$ and $\sigma^k_h(s,a,s^\prime)=\mathds{1}(s^k_h,a^k_h,s^k_{h+1}=s,a,s^\prime)$ by adding independent Laplace noises: $$\widetilde{\sigma}^k_h(s,a)=\sigma^k_h(s,a)+\text{Lap}(\frac{3H}{\epsilon}),$$
    $$\widetilde{\sigma}^k_h(s,a,s^\prime)=\sigma^k_h(s,a,s^\prime)+\text{Lap}(\frac{3H}{\epsilon}).$$
    \item[(2)] The noisy counts are calculated by $$\widehat{N}^k_h(s,a)=\sum_{i=1}^{k-1} \widetilde{\sigma}^i_h(s,a),$$
    $$\widehat{N}^k_h(s,a,s^\prime)=\sum_{i=1}^{k-1} \widetilde{\sigma}^i_h(s,a,s^\prime).$$
    Then the private counts $\widetilde{N}^k_h$ are solved through the procedure in Section \ref{sec:post} with\\ $\e=O(\frac{H}{\epsilon}\sqrt{K\log(HSAT/\beta)})$.
    \item[(3)] We perturb the trajectory-wise reward by adding independent Laplace noise:  $\widetilde{r}^k_h(s,a)=\mathds{1}(s_h^k,a_h^k=s,a)\cdot r^k_h+\text{Lap}(\frac{3H}{\epsilon})$. The accumulative statistic is calculated by $\widetilde{R}^k_h(s,a)=\sum_{i=1}^{k-1}\widetilde{r}^i_h(s,a)$.
\end{enumerate}

Properties of our Local Privatizer is summarized below.

\begin{lemma}\label{lem:local}
For any $\epsilon >0$ and $0<\beta<1$, the Local Privatizer satisfies $\epsilon$-LDP and Assumption \ref{assump} with $\e=\widetilde{O}(\frac{H}{\epsilon}\sqrt{K})$.
\end{lemma}

Therefore, combining Lemma \ref{lem:local} and Theorem \ref{thm:main}, the following regret bound holds.

\begin{theorem}[Regret under LDP]\label{thm:local}
For any $\epsilon>0$ and $0<\beta<1$, running DP-UCBVI (Algorithm \ref{alg:main}) with Local Privatizer as input, with probability $1-\beta$, it holds that:
\begin{equation}
    \text{Regret}(K) \leq \widetilde{O}(\sqrt{SAH^2T}+S^2A\sqrt{H^5T}/\epsilon).
\end{equation}
\end{theorem}

Theorem \ref{thm:local} improves the non-private part of regret bound in the best known result (Corollary 5.5 of \citet{chowdhury2021differentially}).

\subsection{More discussions}
The step (1) of our Privatizers is similar to previous works \citep{vietri2020private,garcelon2021local,chowdhury2021differentially}. However, different from their approaches (directly use $\widehat{N}^k_h$ as private counts), we apply the post-processing step in Section \ref{sec:post}, which ensures that $\widetilde{P}^k_h$ is valid probability distribution while $\e$ is only worse by a constant factor. Therefore, we can apply Bernstein type bonus to achieve the optimal non-private part in our regret bound.

We remark that the Laplace Mechanism can be replaced with other mechanisms, like Gaussian Mechanism \citep{dwork2014algorithmic} for approximate DP (or zCDP). According to Theorem \ref{thm:main}, the regret bounds can be easily derived by plugging in the corresponding $\e$. 


\section{Proof Sketch}\label{sec:sketch}
In this section, we provide a proof overview for Theorem \ref{thm:main}, which can imply the results under JDP (Theorem \ref{thm:central}) and LDP (Theorem \ref{thm:local}). Recall that $N^k_h(s,a)$ and $N^k_h(s,a,s^\prime)$ are real visitation numbers while $\widetilde{N}^k_h$'s are private ones satisfying Assumption \ref{assump}. Other notations like $\widetilde{P}^k_h$, $\widetilde{r}^k_h$, $\widetilde{Q}^k_h$, $\widetilde{V}^k_h$ and $\iota$ are defined in Algorithm \ref{alg:main}. The statement ``with high probability'' means that the summation of all failure probabilities is bounded by $\beta$. We begin with some properties of private statistics below.

\textbf{Properties of $\widetilde{P}$ and $\widetilde{r}$.} Due to concentration inequalities and Assumption \ref{assump}, we provide high probability bounds for $\left|\widetilde{r}^k_h(s,a)-r_h(s,a)\right|$, $\left\|\widetilde{P}^k_h(\cdot|s,a)-P_h(\cdot|s,a)\right\|_1$ and $\left|\widetilde{P}^k_h(s^\prime|s,a)-P_h(s^\prime|s,a)\right|$ in Appendix \ref{sec:app}. In addition, we bound the key term $\left|\left(\p-P_h\right)\cdot V^\star_{h+1}(s,a)\right|$ below.

\begin{lemma}[Informal version of Lemma \ref{lem:pv}]\label{lem:pvpv}
With high probability, for all $(h,s,a,k)$, it holds that:
\begin{equation}
\begin{split}
\left|\left(\p-P_h\right)\cdot V^\star_{h+1}(s,a)\right|\leq&\widetilde{O}\left(\sqrt{\frac{\Var_{\hp(\cdot|s,a)}V^\star_{h+1}(\cdot)}{\N(s,a)}}\right)\\ +&\widetilde{O}\left(\frac{HS\e}{\N(s,a)}\right).
\end{split}
\end{equation}
\end{lemma}

With these concentrations, we are ready to present our proof sketch. Since we apply Bernstein-type bonus, the proof of optimism is not straightforward. We prove our regret upper bound through induction, which is shown below.

\textbf{Induction over episodes.} Our induction is for all $k\in[K]$,
\begin{enumerate}[leftmargin=0cm,itemindent=.5cm,labelwidth=\itemindent,labelsep=0cm,align=left]
\item[(1)] Given that for all $(i,h,s,a)\in[k]\times[H]\times\mathcal{S}\times\mathcal{A}$, $\widetilde{Q}^i_h(s,a)\geq Q^\star_h(s,a)$, we prove ($T_k=kH$) $$\Re(k)\leq\widetilde{O}\left(\sqrt{H^2SAT_k}+H^2S^2A\e\right)$$
and for all $(h,s)\in[H]\times\mathcal{S}$, 
$$\widetilde{V}^{k}_h(s)-V^\star_h(s)\leq \widetilde{O}\left(\sqrt{\frac{SAH^3}{N^k_h(s)}}+\frac{S^2AH^2\e}{N^k_h(s)}\right).$$
\item[(2)] Given that for all $(h,s)\in[H]\times\mathcal{S}$, 
$$\widetilde{V}^{k}_h(s)-V^\star_h(s)\leq \widetilde{O}\left(\sqrt{\frac{SAH^3}{N^k_h(s)}}+\frac{S^2AH^2\e}{N^k_h(s)}\right),$$ we prove that for all $(h,s,a)$, $\widetilde{Q}^{k+1}_h(s,a)\geq Q^\star_h(s,a)$.
\end{enumerate}

Suppose the above induction holds, we have point (1) holds for all $k\in[K]$ and therefore, 
\begin{equation}
    \Re(K)\leq\widetilde{O}\left(\sqrt{H^2SAT}+H^2S^2A\e\right).
\end{equation} 
Below we discuss about the proof of (1) and (2) separately.

\textbf{Proof of regret bound: (1).} We only need to prove the upper bound of $\Re(k)$, as the upper bound of $\widetilde{V}^{k}_h(s)-V^\star_h(s)$ follows similarly. Using the standard technique of layer-wise error decomposition (details in Appendix \ref{sec:ed}) and ignoring lower order terms: summation of martingale differences, we only need to bound $\sum_{i=1}^k\sum_{h=1}^H b^i_h(s^i_h,a^i_h)$ which consists of four terms according to the definition of $b^k_h$. First of all, the second and forth terms are dominated by the first and third terms. Next, for the third term, we have
\begin{equation}
    \begin{split}
        \sum_{i=1}^k\sum_{h=1}^H \frac{20HS\e\iota}{\widetilde{N}^i_h(s^i_h,a^i_h)}\leq&\sum_{i=1}^k\sum_{h=1}^H \frac{20HS\e\iota}{N^i_h(s^i_h,a^i_h)}\\\leq& \widetilde{O}(S^2AH^2\e).
    \end{split}
\end{equation}

Now we analyze the first term (which is also the main term): $\sum_{i=1}^k\sum_{h=1}^H \sqrt{\frac{\Var_{s^\prime\sim\widetilde{P}_h^i(\cdot|s^i_h,a^i_h)}\widetilde{V}^i_{h+1}(\cdot)}{\widetilde{N}^i_h(s^i_h,a^i_h)}}$. It holds that
\begin{equation}
    \begin{split}
        &\sum_{i=1}^k\sum_{h=1}^H \sqrt{\frac{\Var_{s^\prime\sim\widetilde{P}_h^i(\cdot|s^i_h,a^i_h)}\widetilde{V}^i_{h+1}(\cdot)}{\widetilde{N}^i_h(s^i_h,a^i_h)}}\\\leq&
        \sqrt{\underbrace{\sum_{i=1}^k\sum_{h=1}^H \frac{1}{N^i_h(s^i_h,a^i_h)}}_{\leq \widetilde{O}(HSA)}}\cdot\sqrt{\underbrace{\sum_{i=1}^k\sum_{h=1}^H \Var_{\widetilde{P}_h^i(\cdot|s^i_h,a^i_h)}\widetilde{V}^i_{h+1}(\cdot)}_{(a)}}.
    \end{split}
\end{equation}

We bound (a) below (details are deferred to Appendix \ref{sec:variance}).

\begin{equation}
    \begin{split}
        &(a)\leq
        \underbrace{\sum_{i=1}^k\sum_{h=1}^H \Var_{P_h(\cdot|s^i_h,a^i_h)} V^{\pi_i}_{h+1}(\cdot)}_{\leq \widetilde{O}(H^2k)\; \text{w.h.p due to LTV}}+\text{lower order terms}.
    \end{split}
\end{equation}

Therefore, the main term in the regret bound scales as $\widetilde{O}(\sqrt{H^2SAT_k}+H^2S^2A\e)$. The details about lower order terms are deferred to Appendix \ref{sec:ed}, \ref{sec:variance} and \ref{sec:regret}.

\textbf{Proof of optimism: (2).} To prove optimism, we only need 
$$b^{k}_h(s,a)\geq |\widetilde{r}^{k}_h(s,a)-r_h(s,a)|+|(\p-P_h)\cdot V^\star_{h+1}(s,a)|.$$
It is clear that $|\widetilde{r}^k_h(s,a)-r_h(s,a)|$ can be bounded by the second term and a portion of the third term of $b^k_h(s,a)$. Due to Lemma \ref{lem:pvpv}, $|(\p-P_h)\cdot V^\star_{h+1}(s,a)|$ can be bounded by $\widetilde{O}\left(\sqrt{\Var_{\hp(\cdot|s,a)}V^\star_{h+1}(\cdot)/\N(s,a)}\right)$, which can be further bounded by $\widetilde{O}\left(\sqrt{\Var_{\p(\cdot|s,a)}V^\star_{h+1}(\cdot)/\N(s,a)}\right)$ plus a portion of the third term of $b^k_h(s,a)$. Finally, together with the upper bound of $|\widetilde{V}^k_h(s)-V^\star_h(s)|$ (derived from condition of (2) and optimism), the last term of $b^k_h(s,a)$ compensates for the difference of $\sqrt{\Var_{\p(\cdot|s,a)}\widetilde{V}^k_{h+1}(\cdot)/\N(s,a)}$ (first term of $b^k_h(s,a)$) and $\sqrt{\Var_{\p(\cdot|s,a)}V^\star_{h+1}(\cdot)/\N(s,a)}$. More details about optimism are deferred to Appendix \ref{sec:ucb}.


\section{Simulations}\label{sec:simulations}

\begin{figure}[h]\label{fig}
\centering     
	\includegraphics[width=100mm]{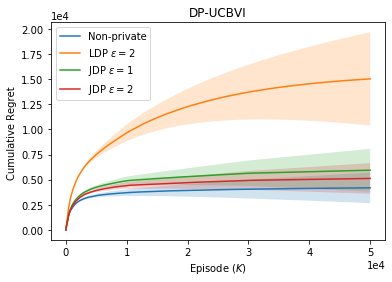}
\caption{Comparison of cumulative regret for UCBVI and DP-UCBVI with different DP guarantees.}
\end{figure}

In this section, we run simulations to show the performance of DP-UCBVI (Algorithm \ref{alg:main}). We run simulation on a standard benchmark for tabular MDP: Riverswim \citep{strehl2008analysis}, and \citet{chowdhury2021differentially} run simulations on the same environment. Briefly speaking,  the environment consists of six consecutive states and two actions ``left'' and ``right''. Choosing ``left'', the agent will tend to move towards the left side, and vice versa. The agent starts from the left side and tries to reach the right side, where she can get higher reward. For more details and illustration about this setting, please refer to \citet{chowdhury2021differentially}. 

Similar to \citet{chowdhury2021differentially}, we set the planning horizon to be $H=20$ and run $K=50000$ episodes. For each algorithm, we run $5$ times and derive the average performance and confidence region. We compare the performance of DP-UCBVI under constraints of JDP and LDP, and the original UCBVI. The cumulative regret for each algorithm is shown in Figure 1. Comparing the regret, it is shown that the non-private UCBVI has the best performance, while the cost of privacy under constraints of JDP is a small constant, and thus becomes negligible as the number of episodes increases. In addition, the DP-UCBVI with weaker privacy protection (i.e., larger $\epsilon$) has smaller regret. However, under constrains of LDP, the cost of privacy remains high and it takes a much longer period for the algorithm to converge to near-optimal policies. Our simulation results are consistent with our theories which state that the cost of JDP is a constant term while the cost of LDP is multiplicative.


\section{Conclusion}\label{sec:conclude}
In this paper, we studied the well-motivated problem of differentially private reinforcement learning. Under the tabular MDP setting, we propose a general framework: DP-UCBVI (Algorithm \ref{alg:main}) that can be combined with any Privatizers for different variants of DP. Under $\epsilon$-JDP, we achieved regret bound of $\widetilde{O}(\sqrt{SAH^2T}+S^2AH^3/\epsilon)$, which matches the lower bound up to lower order terms. Meanwhile, under $\epsilon$-LDP, we derived regret upper bound of $\widetilde{O}(\sqrt{SAH^2T}+S^2A\sqrt{H^5T}/\epsilon)$ and improves the best known result. 

We believe our framework can be further generalized to more general settings, like the linear MDP setting. The best known result under linear MDP \citep{ngo2022improved} built upon LSVI-UCB \citep{jin2020provably}, which is arguably a Hoeffding-type algorithm. The main term of regret bound in \citet{ngo2022improved}, $\widetilde{O}(\sqrt{d^3H^3T})$, is known to be suboptimal due to the recent work \citep{hu2022nearly}, which incorporates Bernstein-type self-normalized concentration. An interesting future direction is to privatize LSVI-UCB$^+$ (Algorithm 1 in \citet{hu2022nearly}) and derive tighter regret bounds under linear MDP and constraints of JDP. We believe the techniques in this paper (privatization of Bernstein-type bonus under tabular MDP) could serve as basic building blocks.

\section*{Acknowledgments}
The research is partially supported by NSF Awards \#2007117 and \#2048091. 

\bibliographystyle{plainnat}
\bibliography{sections/stat_rl}

\newpage
\appendix

\section{Extended related works}\label{sec:erw}
\textbf{Regret minimization under tabular MDP.} Under the most fundamental setting of tabular MDP, regret minimization has been widely studied by a long stream of works \citep{kearns2002near,jaksch2010near,jin2018q,xu2021logarithmic,qiao2022sample,xu2022doubly,qiao2022near,xu2022towards}. Among the optimal results, \citet{azar2017minimax} designed an UCB-based algorithm: UCBVI and derived the minimax optimal regret bound $\widetilde{O}(\sqrt{HSAT})$ under stationary MDP. Later, \citet{zhang2020almost} achieved the optimal regret bound $\widetilde{O}(\sqrt{H^{2}SAT})$  under non-stationary MDP through Q-learning type algorithm: UCB-ADVANTAGE. Meanwhile, in addition to stating optimal regret bound, \citet{dann2019policy} also provided policy certificates via their algorithm: ORLC. Different from the minimax optimal algorithms above, \citet{zanette2019tighter} designed an algorithm: EULER and derived the first problem-dependent regret bound, which can imply the minimax optimal regret. 

\textbf{Other differentially private reinforcement learning algorithms.} In this paragraph, we discuss about algorithms under linear MDP or linear mixture MDP. Under linear MDP, the only algorithm with JDP guarantee: Private LSVI-UCB \citep{ngo2022improved} is private version of LSVI-UCB \citep{jin2020provably}, while LDP under linear MDP still remains open. Under linear mixture MDP, LinOpt-VI-Reg \citep{zhou2022differentially} generalized UCRL-VTR \citep{ayoub2020model} to guarantee JDP. In addition, \citet{liao2021locally} also privatized UCRL-VTR for LDP guarantee. On the offline side, \citet{qiao2022offline} provided the first result under linear MDP based on VAPVI \citep{yin2022near}. 

\section{Properties of private estimations}\label{sec:app}
In this section, we present some useful concentrations about our private estimations that hold with high probability. Throughout the proof, we denote the non-private estimations by:
\begin{equation}
\begin{split}
    \widehat{P}^k_h(s^\prime|s,a)=\frac{N^k_h(s,a,s^\prime)}{N^k_h(s,a)},\\
    \widehat{r}^k_h(s,a)=\frac{R^k_h(s,a)}{N^k_h(s,a)}.
\end{split}    
\end{equation}

In addition, recall that our private estimations are defined as:

\begin{equation}
    \begin{split}
        \widetilde{P}^k_h(s^\prime|s,a)=\frac{\widetilde{N}^k_h(s,a,s^\prime)}{\widetilde{N}^k_h(s,a)},\\
        \widetilde{r}^k_h(s,a)=\left(\frac{\widetilde{R}^k_h(s,a)}{\widetilde{N}^k_h(s,a)}\right)_{[0,1]}.
    \end{split}
\end{equation}

\begin{lemma}\label{lem:r}
With probability $1-\frac{\beta}{15}$, for all $h,s,a,k\in[H]\times\mathcal{S}\times\mathcal{A}\times[K]$, it holds that:
\begin{equation}
    \left|\widetilde{r}^k_h(s,a)-r_h(s,a)\right|\leq \sqrt{\frac{2\iota}{\widetilde{N}^k_h(s,a)}}+\frac{2\e}{\N(s,a)}.
\end{equation}
\end{lemma}

\begin{proof}[Proof of Lemma \ref{lem:r}]
We have for all $h,s,a,k\in[H]\times\mathcal{S}\times\mathcal{A}\times[K]$,
\begin{equation}
    \begin{split}
        &\left|\widetilde{r}^k_h(s,a)-r_h(s,a)\right|\leq\left|\frac{\R(s,a)}{\N(s,a)}-r_h(s,a)\right|\\\leq&\left|\frac{\R(s,a)}{\N(s,a)}-\frac{R^k_h(s,a)}{\N(s,a)}\right|+\left|\frac{R^k_h(s,a)}{\N(s,a)}-r_h(s,a)\right|\\\leq&
        \frac{\e}{\N(s,a)}+\left|\frac{N^k_h(s,a)}{\N(s,a)}\left(\frac{R^k_h(s,a)}{N^k_h(s,a)}-r_h(s,a)\right)\right|+\left|r_h(s,a)\left(\frac{N^k_h(s,a)}{\N(s,a)}-1\right)\right|\\\leq&
        \frac{\e}{\N(s,a)}+\frac{N^k_h(s,a)}{\N(s,a)}\cdot\sqrt{\frac{2\iota}{N^k_h(s,a)}}+\frac{\e}{\N(s,a)}\\\leq&
        \sqrt{\frac{2\iota}{\N(s,a)}}+\frac{2\e}{\N(s,a)},
    \end{split}
\end{equation}
where the third and last inequalities are because of Assumption \ref{assump}. The forth inequality holds with probability $1-\frac{\beta}{15}$ due to Hoeffding's inequality and union bound over $h,s,a,k$.
\end{proof}

\begin{lemma}\label{lem:p1}
With probability $1-\frac{\beta}{15}$, for all $h,s,a,k\in[H]\times\mathcal{S}\times\mathcal{A}\times[K]$, it holds that:
\begin{equation}
    \left\|\p(\cdot|s,a)-P_h(\cdot|s,a)\right\|_1\leq 2\sqrt{\frac{S\iota}{\widetilde{N}^k_h(s,a)}}+\frac{2S\e}{\N(s,a)}.
\end{equation}
\end{lemma}

\begin{proof}[Proof of Lemma \ref{lem:p1}]
We have for all $h,s,a,k\in[H]\times\mathcal{S}\times\mathcal{A}\times[K]$,
\begin{equation}
    \begin{split}
        &\left\|\p(\cdot|s,a)-P_h(\cdot|s,a)\right\|_1=\sum_{s^\prime}\left|\p(s^\prime|s,a)-P_h(s^\prime|s,a)\right|\\\leq&
        \sum_{s^\prime}\left|\frac{\N(s,a,s^\prime)-N^k_h(s,a,s^\prime)}{\N(s,a)}\right|+\sum_{s^\prime}\left|\frac{N^k_h(s,a,s^\prime)}{\N(s,a)}-P_h(s^\prime|s,a)\right|\\\leq&
        \frac{S\e}{\N(s,a)}+\sum_{s^\prime}\left|\frac{N^k_h(s,a,s^\prime)}{N^k_h(s,a)}\cdot\frac{N^k_h(s,a)}{\N(s,a)}-P_h(s^\prime|s,a)\right|\\\leq&
        \frac{S\e}{\N(s,a)}+\sum_{s^\prime}\left|\left(\frac{N^k_h(s,a,s^\prime)}{N^k_h(s,a)}-P_h(s^\prime|s,a)\right)\cdot\frac{N^k_h(s,a)}{\N(s,a)}\right|+\sum_{s^\prime}\left|P_h(s^\prime|s,a)\left(\frac{N^k_h(s,a)}{\N(s,a)}-1\right)\right|\\\leq&
        \frac{S\e}{\N(s,a)}+\frac{N^k_h(s,a)}{\N(s,a)}\left\|\hp(\cdot|s,a)-P_h(\cdot|s,a)\right\|_1+\sum_{s^\prime}\left|P_h(s^\prime|s,a)\frac{\e}{\N(s,a)}\right|\\\leq&
        \frac{N^k_h(s,a)}{\N(s,a)}\cdot 2\sqrt{\frac{S\iota}{N^k_h(s,a)}}+\frac{2S\e}{\N(s,a)}\\\leq&
        2\sqrt{\frac{S\iota}{\N(s,a)}}+\frac{2S\e}{\N(s,a)},
    \end{split}
\end{equation}
where the second, forth and last inequalities hold since Assumption \ref{assump}. The fifth inequality holds with probability $1-\frac{\beta}{15}$ according to Theorem 2.1 of \citet{weissman2003inequalities} and union bound.
\end{proof}

\begin{remark}\label{rem:p1}
Similarly, we have for all $h,s,a,k\in[H]\times\mathcal{S}\times\mathcal{A}\times[K]$,
\begin{equation}
    \begin{split}
        \left\|\p(\cdot|s,a)-\hp(\cdot|s,a)\right\|_1\leq&\sum_{s^\prime}\left|\frac{\N(s,a,s^\prime)}{\N(s,a)}-\frac{N^k_h(s,a,s^\prime)}{\N(s,a)}\right|+\sum_{s^\prime}\left|\frac{N^k_h(s,a,s^\prime)}{\N(s,a)}-\frac{N^k_h(s,a,s^\prime)}{N^k_h(s,a)}\right|\\\leq&\frac{2S\e}{\N(s,a)}.
    \end{split}
\end{equation}
\end{remark}

\begin{lemma}\label{lem:p}
With probability $1-\frac{\beta}{15}$, for all $h,s,a,s^\prime,k\in[H]\times\mathcal{S}\times\mathcal{A}\times\mathcal{S}\times[K]$, it holds that:
\begin{equation}
    \left|\p(s^\prime|s,a)-P_h(s^\prime|s,a)\right|\leq \sqrt{\frac{2P_h(s^\prime|s,a)\iota}{\widetilde{N}^k_h(s,a)}}+\frac{2\e\iota}{\N(s,a)}.
\end{equation}
\end{lemma}

\begin{proof}[Proof of Lemma \ref{lem:p}]
We have for all $h,s,a,s^\prime,k\in[H]\times\mathcal{S}\times\mathcal{A}\times\mathcal{S}\times[K]$,
\begin{equation}
    \begin{split}
        &\left|\p(s^\prime|s,a)-P_h(s^\prime|s,a)\right|\leq\left|\frac{\N(s,a,s^\prime)-N^k_h(s,a,s^\prime)}{\N(s,a)}\right|+\left|\frac{N^k_h(s,a,s^\prime)}{\N(s,a)}-P_h(s^\prime|s,a)\right|\\\leq&
        \frac{\e}{\N(s,a)}+\left|\frac{N^k_h(s,a,s^\prime)}{N^k_h(s,a)}\cdot\frac{N^k_h(s,a)}{\N(s,a)}-P_h(s^\prime|s,a)\right|\\\leq&
        \frac{\e}{\N(s,a)}+\left|\left(\frac{N^k_h(s,a,s^\prime)}{N^k_h(s,a)}-P_h(s^\prime|s,a)\right)\cdot\frac{N^k_h(s,a)}{\N(s,a)}\right|+\left|P_h(s^\prime|s,a)\left(\frac{N^k_h(s,a)}{\N(s,a)}-1\right)\right|\\\leq&
        \frac{2\e}{\N(s,a)}+\frac{N^k_h(s,a)}{\N(s,a)}\cdot\left|\hp(s^\prime|s,a)-P_h(s^\prime|s,a)\right|\\\leq&
        \frac{2\e}{\N(s,a)}+\frac{N^k_h(s,a)}{\N(s,a)}\cdot\left(\sqrt{\frac{2P_h(s^\prime|s,a)\iota}{N^k_h(s,a)}}+\frac{2\iota}{3N^k_h(s,a)}\right)\\\leq&
        \sqrt{\frac{2P_h(s^\prime|s,a)\iota}{\N(s,a)}}+\frac{2\e\iota}{\N(s,a)},
    \end{split}
\end{equation}
where the second, forth and last inequalities result from Assumption \ref{assump}. The fifth inequality holds with probability $1-\frac{\beta}{15}$ due to Bernstein's inequality and union bound.
\end{proof}

\begin{remark}\label{rem:p}
Similarly, we have for all $h,s,a,s^\prime,k\in[H]\times\mathcal{S}\times\mathcal{A}\times\mathcal{S}\times[K]$,
\begin{equation}
    \begin{split}
        \left|\p(s^\prime|s,a)-\hp(s^\prime|s,a)\right|\leq& \left|\frac{\N(s,a,s^\prime)}{\N(s,a)}-\frac{N^k_h(s,a,s^\prime)}{\N(s,a)}\right|+\left|\frac{N^k_h(s,a,s^\prime)}{\N(s,a)}-\frac{N^k_h(s,a,s^\prime)}{N^k_h(s,a)}\right|\\\leq&
        \frac{\e}{\N(s,a)}+\frac{N^k_h(s,a,s^\prime)\e}{\N(s,a)\cdot N^k_h(s,a)}\\\leq&
        \frac{2\e}{\N(s,a)}.
    \end{split}
\end{equation}
\end{remark}

\begin{lemma}\label{lem:pv}
With probability $1-\frac{2\beta}{15}$, for all $h,s,a,k\in[H]\times\mathcal{S}\times\mathcal{A}\times[K]$, it holds that:
\begin{equation}
    \left|\left(\p-P_h\right)\cdot V^\star_{h+1}(s,a)\right|\leq\min\left\{\sqrt{\frac{2\Var_{P_h(\cdot|s,a)}V^\star_{h+1}(\cdot)\cdot\iota}{\N(s,a)}},\sqrt{\frac{2\Var_{\hp(\cdot|s,a)}V^\star_{h+1}(\cdot)\cdot\iota}{\N(s,a)}}\right\}+\frac{2HS\e\iota}{\N(s,a)}.
\end{equation}
\end{lemma}

\begin{proof}[Proof of Lemma \ref{lem:pv}]
We have for all $h,s,a,k\in[H]\times\mathcal{S}\times\mathcal{A}\times[K]$,
\begin{equation}
    \begin{split}
        &\left|\left(\p-P_h\right)\cdot V^\star_{h+1}(s,a)\right|\leq\left|\sum_{s^\prime}\left(\p(s^\prime|s,a)-P_h(s^\prime|s,a)\right)V^\star_{h+1}(s^\prime)\right|\\\leq&
        \left|\sum_{s^\prime}\frac{\N(s,a,s^\prime)-N^k_h(s,a,s^\prime)}{\N(s,a)}\cdot V^\star_{h+1}(s^\prime)\right|+\left|\sum_{s^\prime}\left(\frac{N^k_h(s,a,s^\prime)}{\N(s,a)}-P_h(s^\prime|s,a)\right)V^\star_{h+1}(s^\prime)\right|\\\leq&
        \frac{HS\e}{\N(s,a)}+\left|\frac{N^k_h(s,a)}{\N(s,a)}\cdot\sum_{s^\prime}\left(\frac{N^k_h(s,a,s^\prime)}{N^k_h(s,a)}-P_h(s^\prime|s,a)\right)V^\star_{h+1}(s^\prime)\right|+\left|\sum_{s^\prime}P_h(s^\prime|s,a)V^\star_{h+1}(s^\prime)\left(\frac{N^k_h(s,a)}{\N(s,a)}-1\right)\right|\\\leq&
        \frac{2HS\e}{\N(s,a)}+\left|\frac{N^k_h(s,a)}{\N(s,a)}\cdot\sum_{s^\prime}\left(\frac{N^k_h(s,a,s^\prime)}{N^k_h(s,a)}-P_h(s^\prime|s,a)\right)V^\star_{h+1}(s^\prime)\right|\\\leq&
        \frac{2HS\e}{\N(s,a)}+\frac{N^k_h(s,a)}{\N(s,a)}\cdot\min\left\{\sqrt{\frac{2\Var_{P_h(\cdot|s,a)}V^\star_{h+1}(\cdot)\cdot\iota}{N^k_h(s,a)}}+\frac{2H\iota}{3N^k_h(s,a)},\sqrt{\frac{2\Var_{\hp(\cdot|s,a)}V^\star_{h+1}(\cdot)\cdot\iota}{N^k_h(s,a)}}+\frac{7H\iota}{3N^k_h(s,a)}\right\}\\\leq&
        \min\left\{\sqrt{\frac{2\Var_{P_h(\cdot|s,a)}V^\star_{h+1}(\cdot)\cdot\iota}{\N(s,a)}},\sqrt{\frac{2\Var_{\hp(\cdot|s,a)}V^\star_{h+1}(\cdot)\cdot\iota}{\N(s,a)}}\right\}+\frac{2HS\e\iota}{\N(s,a)},
    \end{split}
\end{equation}
where the third, forth and last inequalities come from Assumption \ref{assump}. The fifth inequality holds with probability $1-\frac{2\beta}{15}$ because of Bernstein's inequality, Empirical Bernstein's inequality and union bound.
\end{proof}

\begin{remark}\label{rem:pv}
Similarly, we have for all $h,s,a,k\in[H]\times\mathcal{S}\times\mathcal{A}\times[K]$,
\begin{equation}
    \begin{split}
        \left|\left(\p-\hp\right)\cdot V^\star_{h+1}(s,a)\right|\leq H\cdot\left\|\p(\cdot|s,a)-\hp(\cdot|s,a)\right\|_1\leq \frac{2HS\e}{\N(s,a)},
    \end{split}
\end{equation}
where the last inequality results from Remark \ref{rem:p1}.
\end{remark}

Combining all the concentrations, we have the following lemma.

\begin{lemma}\label{lem:concentrate}
Under the high probability event that Assumption \ref{assump} holds, with probability at least $1-\frac{\beta}{3}$, the conclusions in Lemma \ref{lem:r}, Lemma \ref{lem:p1}, Lemma \ref{lem:p}, Lemma \ref{lem:pv}, Remark \ref{rem:p1}, Remark \ref{rem:p} and Remark \ref{rem:pv} hold simultaneously.
\end{lemma}

In the following proof, we will prove under the high probability event where Assumption \ref{assump} and Lemma \ref{lem:concentrate} hold. Lastly, we state the following lemma regarding difference of variance.

\begin{lemma}[Lemma C.5 of \citet{qiao2022offline}]\label{lem:qiao}
For any function $V\in\mathbb{R}^{S}$ such that $\|V\|_\infty \leq H$, it holds that
\begin{equation}
   \left|\sqrt{\Var_{\p(\cdot|s,a)}(V)}-\sqrt{\Var_{\hp(\cdot|s,a)}(V)}\right|\leq \sqrt{3}H\cdot\sqrt{\left\|\p(\cdot|s,a)-\hp(\cdot|s,a)\right\|_1}.
\end{equation}
In addition, according to Remark \ref{rem:p1}, the left hand side can be further bounded by
\begin{equation}
    \left|\sqrt{\Var_{\p(\cdot|s,a)}(V)}-\sqrt{\Var_{\hp(\cdot|s,a)}(V)}\right|\leq 3H\sqrt{\frac{S\e}{\N(s,a)}}.
\end{equation}
\end{lemma}

\section{Proof of Theorem \ref{thm:main}}\label{sec:appp}
In this section, we assume the conclusions in Assumption \ref{assump} and Lemma \ref{lem:concentrate} hold and prove the regret bound. 

\subsection{Some preparations}
\subsubsection{Notations}
For all $i,j\in[K]\times[H]$, we define the following variances we will use throughout the proof.
\begin{equation}
    V^\pi_{i,j}=\Var_{P_j(\cdot|s^i_j,a^i_j)}V^{\pi_i}_{j+1}(\cdot).
\end{equation}

\begin{equation}
    V^\star_{i,j}=\Var_{P_j(\cdot|s^i_j,a^i_j)}V^\star_{j+1}(\cdot).
\end{equation}

\begin{equation}
    \widetilde{V}_{i,j}=\Var_{\widetilde{P}^i_j(\cdot|s^i_j,a^i_j)}\widetilde{V}^{i}_{j+1}(\cdot).
\end{equation}

Next, recall the definition of our private bonus $b^k_h$.
\begin{equation}
    \begin{split}
        b^k_h(s,a)=&2\sqrt{\frac{\Var_{s^\prime\sim\widetilde{P}_h^k(\cdot|s,a)}\widetilde{V}^k_{h+1}(\cdot)\cdot\iota}{\widetilde{N}^k_h(s,a)}}+\sqrt{\frac{2\iota}{\widetilde{N}^k_h(s,a)}}+\frac{20HSE_{\epsilon,\beta}\cdot\iota}{\widetilde{N}^k_h(s,a)}\\+&4\sqrt{\iota}\cdot\sqrt{\frac{\sum_{s^\prime}\widetilde{P}^k_h(s^\prime|s,a)\min\left\{\frac{1000^2H^3SA\iota^2}{\widetilde{N}^k_{h+1}(s^\prime)}+\frac{1000^2H^4S^4A^2\e^2\iota^4}{\widetilde{N}^k_{h+1}(s^\prime)^2}+\frac{1000^2H^6S^4A^2\iota^4}{\widetilde{N}^k_{h+1}(s^\prime)^2},H^2\right\}}{\widetilde{N}^k_h(s,a)}}.
    \end{split}
\end{equation}

According to Assumption \ref{assump}, the private visitation numbers will never underestimate the real ones, therefore it holds that

\begin{equation}
    \begin{split}
        b^k_h(s,a)\leq&\underbrace{2\sqrt{\frac{\Var_{s^\prime\sim\widetilde{P}_h^k(\cdot|s,a)}\widetilde{V}^k_{h+1}(\cdot)\cdot\iota}{N^k_h(s,a)}}+\sqrt{\frac{2\iota}{N^k_h(s,a)}}+\frac{20HSE_{\epsilon,\beta}\cdot\iota}{N^k_h(s,a)}}_{\mathrm{b^k_{h,1}(s,a)}}\\+&\underbrace{4\sqrt{\iota}\cdot\sqrt{\frac{\sum_{s^\prime}\widetilde{P}^k_h(s^\prime|s,a)\min\left\{\frac{1000^2H^3SA\iota^2}{N^k_{h+1}(s^\prime)}+\frac{1000^2H^4S^4A^2\e^2\iota^4}{N^k_{h+1}(s^\prime)^2}+\frac{1000^2H^6S^4A^2\iota^4}{N^k_{h+1}(s^\prime)^2},H^2\right\}}{N^k_h(s,a)}}}_{b^k_{h,2}(s,a)}.
    \end{split}
\end{equation}

For the analysis later, we define $\widehat{b}^k_h(s,a):=2b^k_{h,1}(s,a)+b^k_{h,2}(s,a)$.

In addition, we define the following three terms for all $(i,j)\in[K]\times[H]$:
\begin{equation}
    c_{4,i,j}=\frac{H^2S\iota}{N^i_j(s^i_j,a^i_j)},\,\, c_{1,i,j}=\sqrt{\frac{2V_{i,j}^\star\iota}{N^i_j(s^i_j,a^i_j)}},\,\,\widehat{b}_{i,j}=\widehat{b}^i_j(s^i_j,a^i_j).
\end{equation}

\subsubsection{Typical episodes}
Now we define the typical episodes and the typical episodes with respect to $(h,s)\in[H]\times\mathcal{S}$. Briefly speaking, typical episodes ensure that the number of total episodes or visitation number to some state is large enough.

\begin{definition}[Typical episodes]\label{def:typ}
We define the general typical episodes as $[k]_{\text{typ}}=\{i:i\in[k],\,i\geq 250H^2S^2A\iota^2\}$. Also, we define typical episodes with respect to $(h,s)\in[H]\times\mathcal{S}$ as:
$$[k]_{\text{typ},h,s}=\{i\in[k]:N^i_h(s)\geq 250H^2S^2A\iota^2\},$$
where $N^i_h(s)$ is the real visitation number of $(h,s)$ before episode $i$.
\end{definition}

According to Definition \ref{def:typ} above, it is clear that
\begin{equation}
H\cdot\left|[k]/[k]_{typ}\right|\leq 250H^3S^2A\iota^2.
\end{equation}
In the following proof, when we consider summation over episodes, we can consider only the typical episodes since all episodes that are not typical only contribute to a constant term in final regret bound. 


Finally, we define the following summations for all $k,h,s\in[K]\times[H]\times\mathcal{S}$:
\begin{equation}
    C_k=\sum_{i=1}^k \mathds{1}(i\in[k]_{\text{typ}})\sum_{j=1}^{H}(c_{1,i,j}+c_{4,i,j}).
\end{equation}

\begin{equation}
    B_k=\sum_{i=1}^k \mathds{1}(i\in[k]_{\text{typ}})\sum_{j=1}^{H}\widehat{b}_{i,j}.
\end{equation}

\begin{equation}
    C_{k,h,s}=\sum_{i=1}^k \mathds{1}(s_h^i=s,\,i\in[k]_{\text{typ},h,s})\sum_{j=h}^{H}(c_{1,i,j}+c_{4,i,j}).
\end{equation}

\begin{equation}
    B_{k,h,s}=\sum_{i=1}^k \mathds{1}(s_h^i=s,\,i\in[k]_{\text{typ},h,s})\sum_{j=h}^{H}\widehat{b}_{i,j}.
\end{equation}

\subsection{Our induction}
Since we apply Bernstein-type bonus, different from \citet{chowdhury2021differentially}, optimism is not very straightforward. This is because even if $\widetilde{V}^k_h$ is upper bound of $V^\star_h$ and $\p$ is close to $\hp$, $\Var_{\widetilde{P}_h^k(\cdot|s,a)}\widetilde{V}^k_{h+1}(\cdot)$ is not necessarily an upper bound of $\Var_{\hp(\cdot|s,a)}V^\star_{h+1}(\cdot)$. However, we can prove by induction that $\widetilde{V}^k_{h+1}$ is close enough to $V^\star_{h+1}$, and therefore the last term of $b^k_h$ will be sufficiently large to make $\widetilde{V}^k_h$ a valid upper bound of $V^\star_h$. More precisely, our induction is as below:
\begin{enumerate}
    \item Assume for all $(i,h,s,a)\in[k]\times[H]\times\mathcal{S}\times\mathcal{A}$, $\widetilde{Q}^i_h(s,a)\geq Q^\star_h(s,a)$, we prove for all $(h,s)\in[H]\times\mathcal{S}$, 
    $$\widetilde{V}^{k}_h(s)-V^\star_h(s)\leq \widetilde{O}\left(\sqrt{SAH^3/N^k_h(s)}+S^2AH^2\e/N^k_h(s)+S^2AH^3/N^k_h(s)\right).$$
    \item We deduce that the last term of $b^k_h$ compensates for the possible variance difference and for all $(h,s,a)\in[H]\times\mathcal{S}\times\mathcal{A}$, $\widetilde{Q}^{k+1}_h(s,a)\geq Q^\star_h(s,a)$.
\end{enumerate}

Next, we will first prove the point 1 above under optimism in Section \ref{sec:ed}, Section \ref{sec:variance} and Section \ref{sec:regret}, and then prove optimism (point 2 above) based on point 1 in Section \ref{sec:ucb}.

\subsection{Error decomposition}\label{sec:ed}
We define $\delta_{i,h}:=V^\star_h(s_h^i)-V_h^{\pi_i}(s_h^i)$ and $\widetilde{\delta}_{i,h}:=\widetilde{V}^i_h(s^i_h)-V^{\pi_i}_h(s^i_h)$. Now we provide the error decomposition below, based on optimism, for all $(i,h)\in[k]\times[H]$, 
\begin{equation}\label{equ:ed}
    \begin{split}
        &\delta_{i,h}\leq\widetilde{\delta}_{i,h}=\widetilde{V}^i_h(s^i_h)-V_h^{\pi_i}(s^i_h)=\widetilde{Q}^i_h(s^i_h,a^i_h)-Q^{\pi_i}_h(s^i_h,a^i_h)\\\leq& \widetilde{r}^i_h(s^i_h,a^i_h)+\widetilde{P}^i_h\cdot\widetilde{V}^i_{h+1}(s^i_h,a^i_h)+b^i_h(s^i_h,a^i_h)-r_h(s^i_h,a^i_h)-P_h\cdot V^{\pi_i}_{h+1}(s^i_h,a^i_h)\\\leq&
        b^i_h(s^i_h,a^i_h)+\sqrt{\frac{2\iota}{N^i_h(s^i_h,a^i_h)}}+\frac{2\e}{N^i_h(s^i_h,a^i_h)}+(\widetilde{P}^i_h-P_h)\cdot V^\star_{h+1}(s^i_h,a^i_h)+(\widetilde{P}^i_h-P_h)\cdot (\widetilde{V}^i_{h+1}-V^\star_{h+1})(s^i_h,a^i_h)\\+&P_h\cdot(\widetilde{V}^i_{h+1}-V^{\pi_i}_{h+1})(s^i_h,a^i_h)\\\leq&
        b^i_h(s^i_h,a^i_h)+b^i_{h,1}(s^i_h,a^i_h)+c_{1,i,h}+(\widetilde{P}^i_h-P_h)\cdot (\widetilde{V}^i_{h+1}-V^\star_{h+1})(s^i_h,a^i_h)+P_h\cdot(\widetilde{V}^i_{h+1}-V^{\pi_i}_{h+1})(s^i_h,a^i_h)-\frac{2HS\e\iota}{N^i_h(s^i_h,a^i_h)}\\\leq&
        \widehat{b}_{i,h}+c_{1,i,h}+(\widetilde{P}^i_h-P_h)\cdot (\widetilde{V}^i_{h+1}-V^\star_{h+1})(s^i_h,a^i_h)+P_h\cdot(\widetilde{V}^i_{h+1}-V^{\pi_i}_{h+1})(s^i_h,a^i_h)-\frac{2HS\e\iota}{N^i_h(s^i_h,a^i_h)}.
    \end{split}
\end{equation}
The second inequality is because of the definition of $\widetilde{Q}^i_h$. The third inequality results from Lemma \ref{lem:r}. The forth inequality holds since Lemma \ref{lem:pv} and the definition of $b^i_{h,1}$, $c_{1,i,h}$. The last inequality holds due to definition of $\widehat{b}_{i,h}$.

In addition, we have:
\begin{equation}\label{equ:ed1}
    \begin{split}
        &(\widetilde{P}^i_h-P_h)\cdot (\widetilde{V}^i_{h+1}-V^\star_{h+1})(s^i_h,a^i_h)=\sum_{s^\prime}\left(\widetilde{P}^i_h(s^\prime|s^i_h,a^i_h)-P_h(s^\prime|s^i_h,a^i_h)\right)\cdot\left(\widetilde{V}^i_{h+1}(s^\prime)-V^\star_{h+1}(s^\prime)\right)\\\leq&   
        \sum_{s^\prime}\left(\sqrt{\frac{2P_h(s^\prime|s^i_h,a^i_h)\iota}{N^i_h(s^i_h,a^i_h)}}+\frac{2\e\iota}{N^i_h(s^i_h,a^i_h)}\right)\cdot\left(\widetilde{V}^i_{h+1}(s^\prime)-V^\star_{h+1}(s^\prime)\right)\\\leq&
        \sum_{s^\prime}\left(\frac{P_h(s^\prime|s^i_h,a^i_h)}{H}+\frac{H\iota}{N^i_h(s^i_h,a^i_h)}+\frac{2\e\iota}{N^i_h(s^i_h,a^i_h)}\right)\cdot\left(\widetilde{V}^i_{h+1}(s^\prime)-V^\star_{h+1}(s^\prime)\right)\\\leq& \frac{1}{H}P_h\cdot(\widetilde{V}^i_{h+1}-V^{\pi_i}_{h+1})(s^i_h,a^i_h)+\frac{H^2S\iota}{N^i_h(s^i_h,a^i_h)}+\frac{2HS\e\iota}{N^i_h(s^i_h,a^i_h)}.
    \end{split}
\end{equation}

The first inequality is because of Lemma \ref{lem:p}. The second inequality holds since AM-GM inequality. The last inequality results from the fact that $V^\star_{h+1}\geq V^{\pi_i}_{h+1}$.

Plugging \eqref{equ:ed1} into \eqref{equ:ed}, we have:
\begin{equation}\label{equ:ed2}
    \begin{split}
        \delta_{i,h}\leq\widetilde{\delta}_{i,h}\leq& \widehat{b}_{i,h}+c_{1,i,h}+c_{4,i,h}+(1+\frac{1}{H})P_h\cdot(\widetilde{V}^i_{h+1}-V^{\pi_i}_{h+1})(s^i_h,a^i_h)\\=&(1+\frac{1}{H})\widetilde{\delta}_{i,h+1}+\widehat{b}_{i,h}+c_{1,i,h}+c_{4,i,h}+(1+\frac{1}{H})\epsilon_{i,h},
    \end{split}
\end{equation}
where $\epsilon_{i,h}$ is martingale difference that is bounded in $[-H,H]$.

Recursively applying \eqref{equ:ed2}, we have:
\begin{equation}\label{equ:ed3}
    \widetilde{\delta}_{i,h}\leq 3\sum_{j=h}^H \left[\widehat{b}_{i,j}+c_{1,i,j}+c_{4,i,j}+\epsilon_{i,j}\right].
\end{equation}

Summing over episodes, we have
\begin{equation}
\sum_{i=1}^{k}\delta_{i,h}\leq \sum_{i=1}^{k}\widetilde{\delta}_{i,h}\leq 3\sum_{i=1}^k\sum_{j=h}^H\left[\widehat{b}_{i,j}+c_{1,i,j}+c_{4,i,j}+\epsilon_{i,j}\right].    
\end{equation}

According to Azuma-Hoeffding inequality and union bound, we can bound the partial sum of martingale differences below.

\begin{lemma}\label{lem:md}
Let $T_k:=Hk$ be the number of steps until episode $k$. Then with probability $1-\frac{\beta}{12}$, the following inequalities hold for all $k,h,s\in[K]\times[H]\times\mathcal{S}$ and $h^\prime\geq h$:
\begin{equation}
\begin{split}
    \sum_{i=1}^k \mathds{1}(i\in[k]_{\text{typ}})\sum_{j=h}^{H}\epsilon_{i,j}\leq H\sqrt{T_k\iota}.\\
    \sum_{i=1}^k \mathds{1}(s_h^i=s,\,i\in[k]_{\text{typ},h,s})\sum_{j=h^\prime}^{H}\epsilon_{i,j}\leq H\sqrt{HN_h^k(s)\iota}.
\end{split}
\end{equation}
\end{lemma}

We define $U_{k,h}=3\sum_{i=1}^k \mathds{1}(i\in[k]_{\text{typ}})\sum_{j=h}^{H}\left[\widehat{b}_{i,j}+c_{1,i,j}+c_{4,i,j}\right]+3H\sqrt{T_k\iota}$ and
$$U_{k,h,s}=3\sum_{i=1}^k \mathds{1}(s_h^i=s,\,i\in[k]_{\text{typ},h,s})\sum_{j=h}^{H}\left[\widehat{b}_{i,j}+c_{1,i,j}+c_{4,i,j}\right]+3H\sqrt{HN^k_h(s)\iota}.$$
Therefore, combining \eqref{equ:ed3} and Lemma \ref{lem:md}, we have the following key lemma that upper bounds summation of $\delta_{i,j}$. 

\begin{lemma}\label{lem:mdplus}
Under the high probability event in Lemma \ref{lem:md}, for all $h,s\in[H]\times\mathcal{S}$,
\begin{equation}
    \begin{split}
        \sum_{i=1}^k \mathds{1}(i\in[k]_{\text{typ}})\delta_{i,h}\leq\sum_{i=1}^k \mathds{1}(i\in[k]_{\text{typ}})\widetilde{\delta}_{i,h}\leq U_{k,h}\leq U_{k,1},\\
        \sum_{i=1}^k \mathds{1}(i\in[k]_{\text{typ}})\sum_{j=h}^H \widetilde{\delta}_{i,j}\leq HU_{k,1}.
    \end{split}
\end{equation}
At the same time, for all $h,s\in[H]\times\mathcal{S}$ and $j\geq h$,
\begin{equation}
    \begin{split}
        \sum_{i=1}^k \mathds{1}(s_h^i=s,\,i\in[k]_{\text{typ},h,s})\delta_{i,j}\leq \sum_{i=1}^k \mathds{1}(s_h^i=s,\,i\in[k]_{\text{typ},h,s})\widetilde{\delta}_{i,j}\leq U_{k,h,s}, \\ \sum_{i=1}^k \mathds{1}(s_h^i=s,\,i\in[k]_{\text{typ},h,s})\sum_{j=h}^H \widetilde{\delta}_{i,j}\leq HU_{k,h,s}.
    \end{split}
\end{equation}
\end{lemma}

\subsection{Upper bounds of variance}\label{sec:variance}
From the analysis above, it suffices to derive upper bounds for $C_k$, $B_k$, $C_{k,h,s}$ and $B_{k,h,s}$. As a middle step, we discuss several upper bounds about summation of variances. We begin with the following lemma from \citet{azar2017minimax}. Recall that $V^\pi_{i,j}=\Var_{P_j(\cdot|s^i_j,a^i_j)}V^{\pi_i}_{j+1}(\cdot)$ and $T_k=Hk$.

\begin{lemma}[Lemma 8 of \citet{azar2017minimax}]\label{lem:azar}
With probability $1-\frac{\beta}{12}$, for all $k,h,s\in[K]\times[H]\times\mathcal{S}$, it holds that
\begin{equation}
    \begin{split}
        \sum_{i=1}^{k}\mathds{1}(i\in[k]_{\text{typ}})\sum_{j=h}^{H}V_{i,j}^\pi\leq HT_k+2\sqrt{H^4T_k\iota}+H^3\iota\leq 2HT_k,\\
        \sum_{i=1}^k \mathds{1}(s_h^i=s,\,i\in[k]_{\text{typ},h,s})\sum_{j=h}^H V_{i,j}^\pi\leq H^2N^k_h(s)+2\sqrt{H^5N_h^k(s)\iota}+H^3\iota\leq 2H^2N^k_h(s).
    \end{split}
\end{equation}
\end{lemma}

The proof results from a combination of Law of total variance, Freedman's inequality, union bound and our definition of typical episodes, more details can be found in \citet{azar2017minimax}. Next, we provide another upper bound for further bounding $C_k$ (and $C_{k,h,s}$). Recall that $V^\star_{i,j}=\Var_{P_j(\cdot|s^i_j,a^i_j)}V^\star_{j+1}(\cdot)$.

\begin{lemma}\label{lem:vstar}
Under the high probability event in Lemma \ref{lem:md}, it holds that
\begin{equation}
\sum_{i=1}^{k}\mathds{1}(i\in[k]_{\text{typ}})\sum_{j=1}^H \left(V^\star_{i,j}-V_{i,j}^\pi\right)\leq 2H^2U_{k,1}+2H^2\sqrt{T_k\iota}.
\end{equation}
Similarly, under the same high probability event, for all $(h,s)\in[H]\times\mathcal{S}$,
\begin{equation}
    \sum_{i=1}^k \mathds{1}(s_h^i=s,\,i\in[k]_{\text{typ},h,s})\sum_{j=h}^H\left(V^\star_{i,j}-V^\pi_{i,j}\right)\leq 2H^2U_{k,h,s}+2H^2\sqrt{HN^k_h(s)\iota}.
\end{equation}
\end{lemma}

\begin{proof}[Proof of Lemma \ref{lem:vstar}]
We only prove the first conclusion, the second one can be proven in identical way.
\begin{equation}
    \begin{split}
        &\sum_{i=1}^{k}\mathds{1}(i\in[k]_{\text{typ}})\sum_{j=1}^H \left(V^\star_{i,j}-V_{i,j}^\pi\right)\leq\sum_{i=1}^{k}\mathds{1}(i\in[k]_{\text{typ}})\sum_{j=1}^H \E_{s^\prime\sim P_j(\cdot|s^i_j,a^i_j)}\left[V^\star_{j+1}(s^\prime)^2-V^{\pi_i}_{j+1}(s^\prime)^2\right]\\\leq&
        2H\sum_{i=1}^{k}\mathds{1}(i\in[k]_{\text{typ}})\sum_{j=1}^H\E_{s^\prime\sim P_j(\cdot|s^i_j,a^i_j)}\left[V^\star_{j+1}(s^\prime)-V^{\pi_i}_{j+1}(s^\prime)\right]\\\leq&
        2H\left(\sum_{i=1}^{k}\mathds{1}(i\in[k]_{\text{typ}})\sum_{j=1}^H \widetilde{\delta}_{i,j+1}+H\sqrt{T_k \iota}\right)\\\leq&
        2H^2U_{k,1}+2H^2\sqrt{T_k\iota}.
    \end{split}
\end{equation}
The first inequality is because $V^\star_{j+1}\geq V^{\pi_i}_{j+1}$. The second inequality results from the fact that $V^\star_{j+1}+V^{\pi_i}_{j+1}\leq 2H$. The third inequality holds since Lemma \ref{lem:md}. The last inequality is due to Lemma \ref{lem:mdplus}.
\end{proof}

Lastly, we prove the following lemma for bounding $B_k$ (and $B_{k,h,s}$). Recall that $\widetilde{V}_{i,j}=\Var_{\widetilde{P}^i_j(\cdot|s^i_j,a^i_j)}\widetilde{V}^{i}_{j+1}(\cdot)$.

\begin{lemma}\label{lem:vtilde}
Under the high probability event in Lemma \ref{lem:md}, with probability $1-\frac{\beta}{12K}$, it holds that
\begin{equation}
\sum_{i=1}^{k}\mathds{1}(i\in[k]_{\text{typ}})\sum_{j=1}^H \left(\widetilde{V}_{i,j}-V_{i,j}^\pi\right)\leq 2H^2U_{k,1}+8H^2S\sqrt{HAT_k\iota}+6H^3S^2A\e\iota.
\end{equation}
Similarly, under the same high probability event, for all $(h,s)\in[H]\times\mathcal{S}$,
\begin{equation}
    \sum_{i=1}^k \mathds{1}(s_h^i=s,\,i\in[k]_{\text{typ},h,s})\sum_{j=h}^H\left(\widetilde{V}_{i,j}-V^\pi_{i,j}\right)\leq 2H^2U_{k,h,s}+8H^3S\sqrt{AN^k_h(s)\iota}+6H^3S^2A\e\iota.
\end{equation}
\end{lemma}

\begin{proof}[Proof of Lemma \ref{lem:vtilde}]
We only prove the first conclusion, the second one can be proven in identical way. We have
\begin{equation}
    \begin{split}
        &\sum_{i=1}^{k}\mathds{1}(i\in[k]_{\text{typ}})\sum_{j=1}^H \left(\widetilde{V}_{i,j}-V_{i,j}^\pi\right)\leq \underbrace{\sum_{i=1}^{k}\mathds{1}(i\in[k]_{\text{typ}})\sum_{j=1}^H\left[ \E_{s^\prime\sim\widetilde{P}^i_j(\cdot|s^i_j,a^i_j)}\widetilde{V}^i_{j+1}(s^\prime)^2-\E_{s^\prime\sim P_j(\cdot|s^i_j,a^i_j)}\widetilde{V}^i_{j+1}(s^\prime)^2\right]}_{\mathrm{(a)}}\\+&\underbrace{\sum_{i=1}^{k}\mathds{1}(i\in[k]_{\text{typ}})\sum_{j=1}^H \E_{s^\prime\sim P_j(\cdot|s^i_j,a^i_j)}\left[\widetilde{V}^i_{j+1}(s^\prime)^2-V^{\pi_i}_{j+1}(s^\prime)^2\right]}_{\mathrm{(b)}}\\+&\underbrace{\sum_{i=1}^{k}\mathds{1}(i\in[k]_{\text{typ}})\sum_{j=1}^H\left[\left(\E_{s^\prime\sim P_j(\cdot|s^i_j,a^i_j)}V^\star_{j+1}(s^\prime)\right)^2-\left(\E_{s^\prime\sim \widetilde{P}^i_j(\cdot|s^i_j,a^i_j)}V^\star_{j+1}(s^\prime)\right)^2\right]}_{\mathrm{(c)}}.
    \end{split}
\end{equation}
The inequality holds because of direct calculation and the fact that $V^{\pi_i}_{j+1}\leq V^\star_{j+1}\leq\widetilde{V}^i_{j+1}$. Next we bound these terms separately. First of all,
\begin{equation}
    \begin{split}
        &\mathrm{(a)}\leq \sum_{i=1}^{k}\mathds{1}(i\in[k]_{\text{typ}})\sum_{j=1}^H H^2\cdot\left\|\widetilde{P}^i_j(\cdot|s^i_j,a^i_j)-P_j(\cdot|s^i_j,a^i_j)\right\|_1\\\leq&
        \sum_{i=1}^{k}\mathds{1}(i\in[k]_{\text{typ}})\sum_{j=1}^H H^2\left(2\sqrt{\frac{S\iota}{N^i_j(s^i_j,a^i_j)}}+\frac{2S\e}{N^i_j(s^i_j,a^i_j)}\right)\\\leq&
        2H^2S\sqrt{HAT_k\iota}+2H^3S^2A\e\iota.
    \end{split}
\end{equation}
The second inequality comes from Lemma \ref{lem:p1}. The last inequality is because of direct calculation. For $\mathrm{(b)}$, according to Lemma \ref{lem:md}, similar to the proof of Lemma \ref{lem:vstar}, it holds that
\begin{equation}
    \mathrm{(b)}\leq 2H\left(\sum_{i=1}^{k}\mathds{1}(i\in[k]_{\text{typ}})\sum_{j=1}^H \widetilde{\delta}_{i,j+1}+H\sqrt{T_k \iota}\right)\leq 2H^2U_{k,1}+2H^2\sqrt{T_k\iota}.
\end{equation}
Lastly, for $\mathrm{(c)}$, we have
\begin{equation}
    \begin{split}
        \mathrm{(c)}\leq& \sum_{i=1}^{k}\mathds{1}(i\in[k]_{\text{typ}})\sum_{j=1}^H \left[\left(\E_{s^\prime\sim P_j(\cdot|s^i_j,a^i_j)}V^\star_{j+1}(s^\prime)\right)^2-\left(\E_{s^\prime\sim \widehat{P}^i_j(\cdot|s^i_j,a^i_j)}V^\star_{j+1}(s^\prime)\right)^2+2H^2\cdot\left\|\left(\widetilde{P}^i_j-\widehat{P}^i_j\right)(\cdot|s^i_j,a^i_j)\right\|_1\right]\\\leq&
        \sum_{i=1}^{k}\mathds{1}(i\in[k]_{\text{typ}})\sum_{j=1}^H\left(2H\cdot2H\sqrt{\frac{\iota}{N^i_j(s^i_j,a^i_j)}}+2H^2\cdot\frac{2S\e}{N^i_j(s^i_j,a^i_j)}\right)\\\leq&
        4H^2\sqrt{HSAT_k\iota}+4H^3S^2A\e\iota.
    \end{split}
\end{equation}
The second inequality holds with probability $1-\frac{\beta}{12K}$ due to Hoeffding's inequality and Remark \ref{rem:p1}. The last inequality holds because of direct calculation.

Combining the upper bounds of $\mathrm{(a)},\mathrm{(b)},\mathrm{(c)}$, the proof is complete.
\end{proof}

\subsection{Upper bound of regret}\label{sec:regret}
With the upper bounds in Section \ref{sec:variance}, we are ready to bound $C_k$, $B_k$ and the regret. In this section, we assume the high probability events in Lemma \ref{lem:md} and Lemma \ref{lem:azar} hold. We begin with the upper bound of $C_k$ and $C_{k,h,s}$. Recall that $C_k=\sum_{i=1}^k \mathds{1}(i\in[k]_{\text{typ}})\sum_{j=1}^{H}(c_{1,i,j}+c_{4,i,j})$ and $C_{k,h,s}=\sum_{i=1}^k \mathds{1}(s_h^i=s,\,i\in[k]_{\text{typ},h,s})\sum_{j=h}^{H}(c_{1,i,j}+c_{4,i,j})$. 

\begin{lemma}\label{lem:ck}
Under the high probability events in Lemma \ref{lem:md} and Lemma \ref{lem:azar}, we have
\begin{equation}
    C_k=\sum_{i=1}^k \mathds{1}(i\in[k]_{\text{typ}})\sum_{j=1}^{H}(c_{1,i,j}+c_{4,i,j})\leq3\sqrt{H^2SAT_k\iota^2}+2\sqrt{H^3SAU_{k,1}\iota^2}+H^3S^2A\iota^2.
\end{equation}
Similarly, under the same high probability event, for all $h,s\in[H]\times\mathcal{S}$,
\begin{equation}
    C_{k,h,s}\leq3\sqrt{H^3SAN_h^k(s)\iota^2}+2\sqrt{H^3SAU_{k,h,s}\iota^2}+H^3S^2A\iota^2.
\end{equation}
\end{lemma}

\begin{proof}[Proof of Lemma \ref{lem:ck}]
We only prove the first conclusion, the second one can be proven in identical way. We have
\begin{equation}
    C_k=\underbrace{\sum_{i=1}^k \mathds{1}(i\in[k]_{\text{typ}})\sum_{j=1}^{H}c_{1,i,j}}_{\mathrm{(a)}}+\underbrace{\sum_{i=1}^k \mathds{1}(i\in[k]_{\text{typ}})\sum_{j=1}^{H}c_{4,i,j}}_{\mathrm{(b)}}.
\end{equation}
For $\mathrm{(a)}$, due to Cauchy-Schwarz inequality, it holds that
\begin{equation}
    \begin{split}
        &\mathrm{(a)}=\sum_{i=1}^k \mathds{1}(i\in[k]_{\text{typ}})\sum_{j=1}^{H}\sqrt{\frac{2V^\star_{i,j}\iota}{N^i_j(s^i_j,a^i_j)}}\\\leq&
        \sqrt{2\iota}\cdot\sqrt{\underbrace{\sum_{i=1}^k \mathds{1}(i\in[k]_{\text{typ}})\sum_{j=1}^{H}\frac{1}{N^i_j(s^i_j,a^i_j)}}_{\mathrm{(c)}}}\cdot\sqrt{\underbrace{\sum_{i=1}^k \mathds{1}(i\in[k]_{\text{typ}})\sum_{j=1}^{H}V^\star_{i,j}}_{\mathrm{(d)}}}.
    \end{split}
\end{equation}
Due to direct calculation, we have $\mathrm{(c)}\leq HSA\iota$ and in addition,
\begin{equation}
\begin{split}
    \mathrm{(d)}\leq& \sum_{i=1}^k \mathds{1}(i\in[k]_{\text{typ}})\sum_{j=1}^{H} V^\pi_{i,j}+\sum_{i=1}^k \mathds{1}(i\in[k]_{\text{typ}})\sum_{j=1}^{H}\left(V^\star_{i,j}-V^\pi_{i,j}\right)\\\leq&
    2HT_k+2H^2U_{k,1}+2H^2\sqrt{T_k\iota},
\end{split}
\end{equation}
where the second inequality holds due to Lemma \ref{lem:azar} and Lemma \ref{lem:vstar}. Therefore, We have
\begin{equation}\label{equ:cka}
    \mathrm{(a)}\leq 3\sqrt{H^2SAT_k\iota^2}+2\sqrt{H^3SAU_{k,1}\iota^2}.
\end{equation}
For $\mathrm{(b)}$, according to direct calculation, it holds that
\begin{equation}\label{equ:ckb}
\mathrm{(b)}\leq \sum_{i=1}^k \mathds{1}(i\in[k]_{\text{typ}})\sum_{j=1}^{H}\frac{H^2S\iota}{N^i_j(s^i_j,a^i_j)}\leq H^3S^2A\iota^2.
\end{equation}
Combining \eqref{equ:cka} and \eqref{equ:ckb}, the proof is complete.
\end{proof}

Next, we bound $B_k$ and $B_{k,h,s}$. Recall that $B_k=\sum_{i=1}^k \mathds{1}(i\in[k]_{\text{typ}})\sum_{j=1}^{H}\widehat{b}_{i,j}$ and\\ $B_{k,h,s}=\sum_{i=1}^k \mathds{1}(s_h^i=s,\,i\in[k]_{\text{typ},h,s})\sum_{j=h}^{H}\widehat{b}_{i,j}$.

\begin{lemma}\label{lem:bk}
Under the high probability events in Lemma \ref{lem:md}, Lemma \ref{lem:azar} and Lemma \ref{lem:vtilde}, with probability $1-\frac{\beta}{12K}$,
\begin{equation}
    B_k\leq16\sqrt{H^2SAT_k\iota^2}+6\sqrt{H^3SAU_{k,1}\iota^2}+250H^2S^2A\e\iota^2+250H^3S^2A\iota^2.
\end{equation}
Similarly, under the same high probability event, for all $h,s\in[H]\times\mathcal{S}$,
\begin{equation}
    B_{k,h,s}\leq16\sqrt{H^3SAN_h^k(s)\iota^2}+6\sqrt{H^3SAU_{k,h,s}\iota^2}+250H^2S^2A\e\iota^2+250H^3S^2A\iota^2.
\end{equation}
\end{lemma}

\begin{proof}[Proof of Lemma \ref{lem:bk}]
We only prove the first conclusion, the second one can be proven in identical way. We have
\begin{equation}
\begin{split}
B_k\leq& \underbrace{\sum_{i=1}^k \mathds{1}(i\in[k]_{\text{typ}})\sum_{j=1}^{H}4\sqrt{\frac{\widetilde{V}_{i,j}\iota}{N^i_j(s^i_j,a^i_j)}}}_{\mathrm{(a)}}+\underbrace{\sum_{i=1}^k \mathds{1}(i\in[k]_{\text{typ}})\sum_{j=1}^{H}2\sqrt{\frac{2\iota}{N^i_j(s^i_j,a^i_j)}}}_{\mathrm{(b)}}+\underbrace{\sum_{i=1}^k \mathds{1}(i\in[k]_{\text{typ}})\sum_{j=1}^{H}\frac{40HS\e\iota}{N^i_j(s^i_j,a^i_j)}}_{\mathrm{(c)}} \\+&
\underbrace{\sum_{i=1}^k \mathds{1}(i\in[k]_{\text{typ}})\sum_{j=1}^{H}4\sqrt{\iota}\cdot\sqrt{\frac{\sum_{s^\prime}\widetilde{P}^i_j(s^\prime|s^i_j,a^i_j)\min\left\{\frac{1000^2H^3SA\iota^2}{N^i_{j+1}(s^\prime)}+\frac{1000^2H^4S^4A^2\e^2\iota^4}{N^i_{j+1}(s^\prime)^2}+\frac{1000^2H^6S^4A^2\iota^4}{N^i_{j+1}(s^\prime)^2},H^2\right\}}{N^i_j(s^i_j,a^i_j)}}}_{\mathrm{(d)}}.
\end{split}
\end{equation}
We bound these terms separately, we first bound $\mathrm{(a)}$ below.
\begin{equation}\label{equ:bka}
    \begin{split}
        &\mathrm{(a)}\leq 4\sqrt{\iota}\cdot\sqrt{\sum_{i=1}^k \mathds{1}(i\in[k]_{\text{typ}})\sum_{j=1}^{H}\frac{1}{N^i_j(s^i_j,a^i_j)}}\cdot\sqrt{\sum_{i=1}^k \mathds{1}(i\in[k]_{\text{typ}})\sum_{j=1}^{H}\widetilde{V}_{i,j}}\\\leq&
        4\sqrt{\iota}\cdot\sqrt{HSA\iota}\cdot\sqrt{\sum_{i=1}^k \mathds{1}(i\in[k]_{\text{typ}})\sum_{j=1}^{H}V^\pi_{i,j}+\sum_{i=1}^k \mathds{1}(i\in[k]_{\text{typ}})\sum_{j=1}^{H}\left(\widetilde{V}_{i,j}-V^\pi_{i,j}\right)}\\\leq&
        4\iota\cdot\sqrt{HSA}\sqrt{2HT_k+2H^2U_{k,1}+8H^2S\sqrt{HAT_k\iota}+6H^3S^2A\e\iota}\\\leq&
        8\sqrt{H^2SAT_k\iota^2}+4\sqrt{2H^3SAU_{k,1}\iota^2}+10\sqrt{H^4S^3A^2\e\iota^3}.
    \end{split}
\end{equation}
The first inequality is because of Cauchy-Schwarz inequality. The second inequality results from direct calculation. The third inequality is due to Lemma \ref{lem:azar} and Lemma \ref{lem:vtilde}. The last inequality holds because for typical episodes, $8H^2S\sqrt{HAT_k\iota}\leq 2HT_k$.

According to direct calculation, 
\begin{equation}\label{equ:bkb}
   \mathrm{(b)}\leq 2\sqrt{2\iota\cdot HSAT_k}\leq\sqrt{H^2SAT_k\iota^2} 
\end{equation} 

For $\mathrm{(c)}$, it holds that

\begin{equation}\label{equ:bkc}
    \begin{split}
        \mathrm{(c)}\leq 40HS\e\iota\cdot HSA\iota\leq40H^2S^2A\e\iota^2.
    \end{split}
\end{equation}
Finally, we bound the most complex term $\mathrm{(d)}$ as below. Because of Cauchy-Schwarz inequality,
\begin{equation}
\begin{split}
   &\mathrm{(d)}\leq \sum_{i=1}^k \mathds{1}(i\in[k]_{\text{typ}})\sum_{j=1}^{H}4\sqrt{\iota}\cdot\sqrt{\frac{\sum_{s^\prime}\widetilde{P}^i_j(s^\prime|s^i_j,a^i_j)\min\left\{\frac{1000^2H^3SA\iota^2}{N^i_{j+1}(s^\prime)}+\frac{1000^2H^4S^4A^2\e^2\iota^4}{N^i_{j+1}(s^\prime)^2}+\frac{1000^2H^6S^4A^2\iota^4}{N^i_{j+1}(s^\prime)^2},H^2\right\}}{N^i_j(s^i_j,a^i_j)}}\\\leq& 4\sqrt{\iota\underbrace{\sum_{i=1}^k \mathds{1}(i\in[k]_{\text{typ}})\sum_{j=1}^{H}\frac{1}{N^i_j(s^i_j,a^i_j)}}_{\mathrm{(e)}}}\cdot\\&\sqrt{\underbrace{\sum_{i=1}^k \mathds{1}(i\in[k]_{\text{typ}})\sum_{j=1}^{H}\sum_{s^\prime}\widetilde{P}^i_j(s^\prime|s^i_j,a^i_j)\min\left\{\frac{1000^2H^3SA\iota^2}{N^i_{j+1}(s^\prime)}+\frac{1000^2H^4S^4A^2\e^2\iota^4}{N^i_{j+1}(s^\prime)^2}+\frac{1000^2H^6S^4A^2\iota^4}{N^i_{j+1}(s^\prime)^2},H^2\right\}}_{\mathrm{(f)}}}.
\end{split}
\end{equation}
Note that $\mathrm{(e)}\leq HSA\iota$. For $\mathrm{(f)}$, we have with probability $1-\frac{\beta}{12K}$,
\begin{equation}
    \begin{split}
        &\mathrm{(f)}\leq \sum_{i=1}^k \mathds{1}(i\in[k]_{\text{typ}})\sum_{j=1}^{H} H^2\left\|\widetilde{P}^i_j(\cdot|s^i_j,a^i_j)-P_j(\cdot|s^i_j,a^i_j)\right\|_1\\+&\sum_{i=1}^k \mathds{1}(i\in[k]_{\text{typ}})\sum_{j=1}^{H}\sum_{s^\prime}P_j(s^\prime|s^i_j,a^i_j)\min\left\{\frac{1000^2H^3SA\iota^2}{N^i_{j+1}(s^\prime)}+\frac{1000^2H^4S^4A^2\e^2\iota^4}{N^i_{j+1}(s^\prime)^2}+\frac{1000^2H^6S^4A^2\iota^4}{N^i_{j+1}(s^\prime)^2},H^2\right\}\\\leq&
        \sum_{i=1}^k \mathds{1}(i\in[k]_{\text{typ}})\sum_{j=1}^{H} H^2\left\|\widetilde{P}^i_j(\cdot|s^i_j,a^i_j)-P_j(\cdot|s^i_j,a^i_j)\right\|_1+H^2\sqrt{T_k\iota}\\+&\sum_{i=1}^k \mathds{1}(i\in[k]_{\text{typ}})\sum_{j=1}^{H}\min\left\{\frac{1000^2H^3SA\iota^2}{N^i_{j+1}(s^i_{j+1})}+\frac{1000^2H^4S^4A^2\e^2\iota^4}{N^i_{j+1}(s^i_{j+1})^2}+\frac{1000^2H^6S^4A^2\iota^4}{N^i_{j+1}(s^i_{j+1})^2},H^2\right\}\\\leq& \sum_{i=1}^k \mathds{1}(i\in[k]_{\text{typ}})\sum_{j=1}^{H}H^2\left(2\sqrt{\frac{S\iota}{N^i_j(s^i_j,a^i_j)}}+\frac{2S\e}{N^i_j(s^i_j,a^i_j)}\right)+\sum_{i=1}^k \mathds{1}(i\in[k]_{\text{typ}})\sum_{j=1}^{H}\min\left\{\frac{1000^2H^3SA\iota^2}{N^i_{j+1}(s^i_{j+1})},H^2\right\}\\+&\sum_{i=1}^k \mathds{1}(i\in[k]_{\text{typ}})\sum_{j=1}^{H}\min\left\{\frac{1000^2H^4S^4A^2\e^2\iota^4}{N^i_{j+1}(s^i_{j+1})^2},H^2\right\}+\sum_{i=1}^k \mathds{1}(i\in[k]_{\text{typ}})\sum_{j=1}^{H}\min\left\{\frac{1000^2H^6S^4A^2\iota^4}{N^i_{j+1}(s^i_{j+1})^2},H^2\right\}\\+& H^2\sqrt{T_k\iota} \\\leq&3H^2S\sqrt{HAT_k\iota}+2H^3S^2A\e\iota+1000H^4S^3A\e\iota^2+2000H^5S^3A\iota^2, 
    \end{split}
\end{equation}
where the first inequality holds because $\min\{\cdot,H^2\}\leq H^2$. The second inequality holds with probability $1-\frac{\beta}{12K}$ due to Azuma-Hoeffding inequality. The third inequality results from Lemma \ref{lem:p1}. The last inequality comes from direct calculation. Therefore, we have
\begin{equation}\label{equ:bkd}
    \mathrm{(d)}\leq 7\sqrt{H^3S^2A\iota^2\sqrt{HAT_k\iota}}+120\sqrt{H^5S^4A^2\e\iota^4}+200H^3S^2A\iota^2.
\end{equation}
Combining \eqref{equ:bka}, \eqref{equ:bkb}, \eqref{equ:bkc} and \eqref{equ:bkd}, the proof is complete.
\end{proof}

Now we are ready to bound the regret until episode $k$ based on optimism. We define the following regret functions.
\begin{equation}
    \Re(k):=\sum_{i=1}^k \delta_{i,1},\,\,\widetilde{\Re}(k):=\sum_{i=1}^k\widetilde{\delta}_{i,1}.
\end{equation}

In addition, we define the regret with respect to $(h,s)\in[H]\times\mathcal{S}$:

\begin{equation}
    \Re(k,h,s):=\sum_{i=1}^k \mathds{1}(s^i_h=s)\cdot\delta_{i,h},\,\,\widetilde{\Re}(k,h,s):=\sum_{i=1}^k \mathds{1}(s^i_h=s)\cdot\widetilde{\delta}_{i,h}.
\end{equation}

\begin{lemma}\label{lem:regret}
With probability $1-\beta$, for all $k\in[K]$, as well as the optimism holds (\emph{i.e.} for all $(i,h,s,a)\in[k]\times[H]\times\mathcal{S}\times\mathcal{A}$, $\widetilde{Q}^i_h(s,a)\geq Q^\star_h(s,a)$), it holds that 
\begin{equation}
    \Re(k)\leq\widetilde{\Re}(k)\leq1000\left(\sqrt{H^2SAT_k\iota^2}+H^2S^2A\e\iota^2+H^3S^2A\iota^2\right).
\end{equation}
In addition, for all $(h,s)\in[H]\times\mathcal{S}$, we have
\begin{equation}
    \Re(k,h,s)\leq\widetilde{\Re}(k,h,s)\leq1000\left(\sqrt{H^3SAN^k_h(s)\iota^2}+H^2S^2A\e\iota^2+H^3S^2A\iota^2\right).
\end{equation}
\end{lemma}

\begin{proof}[Proof of Lemma \ref{lem:regret}]
For the proof of this lemma, we assume the high probability events in Assumption \ref{assump}, Lemma \ref{lem:concentrate}, Lemma \ref{lem:md}, Lemma \ref{lem:azar}, Lemma \ref{lem:vtilde} (for all $k\in[K]$) and Lemma \ref{lem:bk} (for all $k\in[K]$) hold. The failure probability is bounded by 
\begin{equation}
    \frac{\beta}{3}+\frac{\beta}{3}+\frac{\beta}{12}+\frac{\beta}{12}+K\cdot\frac{\beta}{12K}+K\cdot\frac{\beta}{12K}\leq\beta.
\end{equation}
We only prove the first conclusion, the second one can be proven in identical way. It holds that
\begin{equation}
\begin{split}
    &\Re(k)\leq\widetilde{\Re}(k)=\sum_{i=1}^k\widetilde{\delta}_{i,1}\\\leq& U_{k,1}+250H^3S^2A\iota^2\\\leq&
    3B_k+3C_k+3H\sqrt{T_k\iota}+250H^3S^2A\iota^2\\\leq&
    60\sqrt{H^2SAT_k\iota^2}+24\sqrt{H^3SAU_{k,1}\iota^2}+750H^2S^2A\e\iota^2+1000H^3S^2A\iota^2\\\leq&
    1000\left(\sqrt{H^2SAT_k\iota^2}+H^2S^2A\e\iota^2+H^3S^2A\iota^2\right),
\end{split}
\end{equation}
where the second inequality is because Lemma \ref{lem:mdplus} and the fact that $H\cdot\left|[k]/[k]_{typ}\right|\leq 250H^3S^2A\iota^2$. The forth inequality is by combining Lemma \ref{lem:bk} and Lemma \ref{lem:ck}. The last inequality results from solving the inequality with respect to $U_{k,1}$.
\end{proof}

\begin{corollary}
Under the event in Lemma \ref{lem:regret}, we have
\begin{equation}
    \begin{split}
        &1000\left(\sqrt{H^3SAN^k_h(s)\iota^2}+H^2S^2A\e\iota^2+H^3S^2A\iota^2\right)\geq \widetilde{\Re}(k,h,s)\\\geq&
        \sum_{i=1}^k \mathds{1}(s^i_h=s)\cdot\widetilde{\delta}_{i,h}\\\geq&\sum_{i=1}^k \mathds{1}(s^i_h=s)\left(\widetilde{V}^i_h(s)-V^{\pi_i}_h(s)\right)\\\geq&
        \sum_{i=1}^k \mathds{1}(s^i_h=s)\left(\widetilde{V}^i_h(s)-V^{\star}_h(s)\right)\\\geq&
        N^k_h(s)\cdot\left(\widetilde{V}^k_h(s)-V^\star_h(s)\right),
    \end{split}
\end{equation}
where the first three inequalities are due to definitions of $\widetilde{\Re}(k,h,s)$ and $\widetilde{\delta}_{i,h}$. The forth inequality holds because $V^\star_h\geq V^{\pi_i}_h$. The last inequality results from our algorithmic design that $\widetilde{V}^i_h(s)$ is non-increasing (line 9 of Algorithm \ref{alg:main}).

Therefore, we have $\widetilde{V}^k_h(s)-V^\star_h(s)\leq1000\left(\sqrt{H^3SA\iota^2/N^k_h(s)}+H^2S^2A\e\iota^2/N^k_h(s)+H^3S^2A\iota^2/N^k_h(s)\right)$.
\end{corollary}

Now we have proven the first point of our induction. Together with the point 2 (which we will prove in Section \ref{sec:ucb}), we have with probability $1-\beta$, the whole induction process is valid. 

For clarity, we restate the induction process under the high probability event in Lemma \ref{lem:regret}. For all $k\in[K]$,\\
1. Given that for all $(i,h,s,a)\in[k]\times[H]\times\mathcal{S}\times\mathcal{A}$, $\widetilde{Q}^i_h(s,a)\geq Q^\star_h(s,a)$, we prove $$\Re(k)\leq\widetilde{\Re}(k)\leq1000\left(\sqrt{H^2SAT_k\iota^2}+H^2S^2A\e\iota^2+H^3S^2A\iota^2\right)$$
and for all $(h,s)\in[H]\times\mathcal{S}$, 
$$\widetilde{V}^{k}_h(s)-V^\star_h(s)\leq 1000\left(\sqrt{SAH^3\iota^2/N^k_h(s)}+S^2AH^2\e\iota^2/N^k_h(s)+S^2AH^3\iota^2/N^k_h(s)\right).$$
2. Given that for all $(h,s)\in[H]\times\mathcal{S}$, \\
$\widetilde{V}^{k}_h(s)-V^\star_h(s)\leq 1000\left(\sqrt{SAH^3\iota^2/N^k_h(s)}+S^2AH^2\e\iota^2/N^k_h(s)+S^2AH^3\iota^2/N^k_h(s)\right),$ we prove that for all $(h,s,a)\in[H]\times\mathcal{S}\times\mathcal{A}$, $\widetilde{Q}^{k+1}_h(s,a)\geq Q^\star_h(s,a)$.

Therefore, with probability $1-\beta$, we have ($T=KH$)
\begin{equation}
\Re(K)\leq\widetilde{\Re}(K)\leq\widetilde{O}\left(\sqrt{H^2SAT}+H^2S^2A\e+H^3S^2A\right).    
\end{equation} 
This completes the proof of Theorem \ref{thm:main}.

\subsection{Proof of optimism}\label{sec:ucb}
In this part, we prove optimism. Given the condition that 
for all $(h,s)\in[H]\times\mathcal{S}$,
$$\widetilde{V}^{k+1}_h(s)-V^\star_h(s)\leq 1000\left(\sqrt{SAH^3\iota^2/N^{k+1}_h(s)}+S^2AH^2\e\iota^2/N^{k+1}_h(s)+S^2AH^3\iota^2/N^{k+1}_h(s)\right),$$
we prove for all $(h,s,a)\in[H]\times\mathcal{S}\times\mathcal{A}$, $\widetilde{Q}^{k+1}_h(s,a)\geq Q^\star_h(s,a)$ through backward induction (induction from $H+1$ to $1$). Since the conclusion holds trivially for $H+1$, it suffices to prove the following Lemma \ref{lem:ucb}.

\begin{lemma}\label{lem:ucb}
Under the high probability event in Assumption \ref{assump} and Lemma \ref{lem:concentrate}, if it holds that 
\begin{enumerate}
    \item For all $(j,s,a)\in[H]\times\mathcal{S}\times\mathcal{A}$, $\widetilde{Q}^k_j(s,a)\geq Q^\star_j(s,a)$.
    \item For all $s\in\mathcal{S}$,\\ $0\leq\widetilde{V}^{k+1}_{h+1}(s)-V^\star_{h+1}(s)\leq 1000\left(\sqrt{SAH^3\iota^2/N^{k+1}_{h+1}(s)}+S^2AH^2\e\iota^2/N^{k+1}_{h+1}(s)+S^2AH^3\iota^2/N^{k+1}_{h+1}(s)\right)$.
\end{enumerate}
Then we have for all $(s,a)\in\mathcal{S}\times\mathcal{A}$, $\widetilde{Q}^{k+1}_h(s,a)\geq Q^\star_h(s,a)$.
\end{lemma}

\begin{proof}[Proof of Lemma \ref{lem:ucb}]
For all $(s,a)\in\mathcal{S}\times\mathcal{A}$, since $\widetilde{Q}^k_h(s,a),H\geq Q^\star_h(s,a)$, it suffices to prove that $\widetilde{r}^{k+1}_h(s,a)+\widetilde{P}^{k+1}_h\cdot\widetilde{V}^{k+1}_{h+1}(s,a)+b^{k+1}_h(s,a)\geq Q^\star_h(s,a)$. We have
\begin{equation}\label{equ:ucb1}
    \begin{split}
        &\widetilde{r}^{k+1}_h(s,a)+\widetilde{P}^{k+1}_h\cdot\widetilde{V}^{k+1}_{h+1}(s,a)+b^{k+1}_h(s,a)-Q^\star_h(s,a)\\\geq&
        \left(\widetilde{r}^{k+1}_h-r_h\right)(s,a)+\left(\widetilde{P}^{k+1}_h-P_h\right)\cdot V^\star_{h+1}(s,a)+b^{k+1}_h(s,a)\\\geq&\left(\widetilde{P}^{k+1}_h-P_h\right)\cdot V^\star_{h+1}(s,a)+
        2\sqrt{\frac{\Var_{s^\prime\sim\widetilde{P}_h^{k+1}(\cdot|s,a)}\widetilde{V}^{k+1}_{h+1}(\cdot)\cdot\iota}{\widetilde{N}^{k+1}_h(s,a)}}+\frac{19HSE_{\epsilon,\beta}\cdot\iota}{\widetilde{N}^{k+1}_h(s,a)}\\+&4\sqrt{\iota}\cdot\sqrt{\frac{\sum_{s^\prime}\widetilde{P}^{k+1}_h(s^\prime|s,a)\min\left\{\frac{1000^2H^3SA\iota^2}{\widetilde{N}^{k+1}_{h+1}(s^\prime)}+\frac{1000^2H^4S^4A^2\e^2\iota^4}{\widetilde{N}^{k+1}_{h+1}(s^\prime)^2}+\frac{1000^2H^6S^4A^2\iota^4}{\widetilde{N}^{k+1}_{h+1}(s^\prime)^2},H^2\right\}}{\widetilde{N}^k_h(s,a)}}\\\geq&
        -\sqrt{\frac{2\Var_{s^\prime\sim\widehat{P}_h^{k+1}(\cdot|s,a)}{V}^{\star}_{h+1}(\cdot)\cdot\iota}{\widetilde{N}^{k+1}_h(s,a)}}+2\sqrt{\frac{\Var_{s^\prime\sim\widetilde{P}_h^{k+1}(\cdot|s,a)}\widetilde{V}^{k+1}_{h+1}(\cdot)\cdot\iota}{\widetilde{N}^{k+1}_h(s,a)}}+\frac{17HSE_{\epsilon,\beta}\cdot\iota}{\widetilde{N}^{k+1}_h(s,a)}\\+&4\sqrt{\iota}\cdot\sqrt{\frac{\sum_{s^\prime}\widetilde{P}^{k+1}_h(s^\prime|s,a)\min\left\{\frac{1000^2H^3SA\iota^2}{\widetilde{N}^{k+1}_{h+1}(s^\prime)}+\frac{1000^2H^4S^4A^2\e^2\iota^4}{\widetilde{N}^{k+1}_{h+1}(s^\prime)^2}+\frac{1000^2H^6S^4A^2\iota^4}{\widetilde{N}^{k+1}_{h+1}(s^\prime)^2},H^2\right\}}{\widetilde{N}^k_h(s,a)}}\\\geq&
        -\sqrt{\frac{2\Var_{s^\prime\sim\widehat{P}_h^{k+1}(\cdot|s,a)}{V}^{\star}_{h+1}(\cdot)\cdot\iota}{\widetilde{N}^{k+1}_h(s,a)}}+2\sqrt{\frac{\Var_{s^\prime\sim\widehat{P}_h^{k+1}(\cdot|s,a)}\widetilde{V}^{k+1}_{h+1}(\cdot)\cdot\iota}{\widetilde{N}^{k+1}_h(s,a)}}\\+&4\sqrt{\iota}\cdot\sqrt{\frac{\sum_{s^\prime}\widehat{P}^{k+1}_h(s^\prime|s,a)\min\left\{\frac{1000^2H^3SA\iota^2}{\widetilde{N}^{k+1}_{h+1}(s^\prime)}+\frac{1000^2H^4S^4A^2\e^2\iota^4}{\widetilde{N}^{k+1}_{h+1}(s^\prime)^2}+\frac{1000^2H^6S^4A^2\iota^4}{\widetilde{N}^{k+1}_{h+1}(s^\prime)^2},H^2\right\}}{\widetilde{N}^k_h(s,a)}}\\\geq&0,
    \end{split}
\end{equation}
where the first inequality is because of Bellman equation and condition 2. The second inequality holds since the definition of $b^{k+1}_h$ and Lemma \ref{lem:r}. The third inequality results from Lemma \ref{lem:pv}. The forth inequality comes from Lemma \ref{lem:qiao} and Remark \ref{rem:p1}. The last inequality is due to the following analysis. 

Because $\sqrt{\Var(X)}\leq\sqrt{2\Var(Y)}+\sqrt{2\Var(X-Y)}$ (Lemma 2 of \citet{azar2017minimax}), we have
\begin{equation}\label{equ:ucb2}
    \begin{split}
        &\sqrt{\Var_{s^\prime\sim\widehat{P}_h^{k+1}(\cdot|s,a)}{V}^{\star}_{h+1}(\cdot)}\leq \sqrt{2\Var_{s^\prime\sim\widehat{P}_h^{k+1}(\cdot|s,a)}\widetilde{V}^{k+1}_{h+1}(\cdot)}+\underbrace{\sqrt{2\Var_{s^\prime\sim\widehat{P}_h^{k+1}(\cdot|s,a)}\left(\widetilde{V}^{k+1}_{h+1}(\cdot)-V^\star_{h+1}(\cdot)\right)}}_{\mathrm{(a)}}.
    \end{split}
\end{equation}
In addition,
\begin{equation}\label{equ:ucb3}
    \begin{split}
        &\mathrm{(a)}\leq \sqrt{2\sum_{s^\prime}\widehat{P}_h^{k+1}(s^\prime|s,a)\left(\widetilde{V}^{k+1}_{h+1}(s^\prime)-V^\star_{h+1}(s^\prime)\right)^2}\\\leq&\sqrt{6\sum_{s^\prime}\widehat{P}_h^{k+1}(s^\prime|s,a)\min\left\{\frac{1000^2H^3SA\iota^2}{N^{k+1}_{h+1}(s^\prime)}+\frac{1000^2H^4S^4A^2\e^2\iota^4}{N^{k+1}_{h+1}(s^\prime)^2}+\frac{1000^2H^6S^4A^2\iota^4}{N^{k+1}_{h+1}(s^\prime)^2},H^2\right\}}\\\leq&2\sqrt{2\sum_{s^\prime}\widehat{P}_h^{k+1}(s^\prime|s,a)\min\left\{\frac{1000^2H^3SA\iota^2}{\widetilde{N}^{k+1}_{h+1}(s^\prime)}+\frac{1000^2H^4S^4A^2\e^2\iota^4}{\widetilde{N}^{k+1}_{h+1}(s^\prime)^2}+\frac{1000^2H^6S^4A^2\iota^4}{\widetilde{N}^{k+1}_{h+1}(s^\prime)^2},H^2\right\}},
    \end{split}
\end{equation}
where the first inequality results from the definition of variance. The second inequality holds because of condition 2 and the fact that $\min\{(a+b+c)^2,H^2\}\leq3\min\{a^2+b^2+c^2,H^2\}$. The last inequality holds according to Assumption \ref{assump}.

Finally, plugging \eqref{equ:ucb2} and \eqref{equ:ucb3} into \eqref{equ:ucb1}, we have the last inequality of \eqref{equ:ucb1} holds. Therefore, the proof of Lemma \ref{lem:ucb} is complete.
\end{proof}

\section{Missing proofs in Section \ref{sec:privatizer}}\label{sec:proofprivate}

In this section, we state the missing proofs in Section \ref{sec:privatizer}. Recall that $N^k_h$ is the original count, $\widehat{N}^k_h$ is the noisy count after step (1) of both Privatizers and $\widetilde{N}^k_h$ is the final private counts. 

\begin{proof}[Proof of Lemma \ref{lem:central}]
Due to Theorem 3.5 of \citet{chan2011private} and Lemma 34 of \citet{hsu2014private}, the release of $\{\widehat{N}^k_h(s,a)\}_{(h,s,a,k)}$ satisfies $\frac{\epsilon}{3}$-DP. Similarly, the releases of $\{\widehat{N}^k_h(s,a,s^\prime)\}_{(k,h,s,a,s^\prime)}$ and $\{\widetilde{R}^k_h(s,a)\}_{(k,h,s,a)}$ both satisfy $\frac{\epsilon}{3}$-DP. Therefore, the release of the following private counters $\{\widehat{N}^k_h(s,a)\}_{(h,s,a,k)}$, $\{\widehat{N}^k_h(s,a,s^\prime)\}_{(k,h,s,a,s^\prime)}$ and $\{\widetilde{R}^k_h(s,a)\}_{(k,h,s,a)}$ satisfy $\epsilon$-DP. Due to post-processing (Lemma 2.3 of \citet{bun2016concentrated}), the release of all private counts $\{\widetilde{N}^k_h(s,a)\}_{(h,s,a,k)}$, $\{\widetilde{N}^k_h(s,a,s^\prime)\}_{(k,h,s,a,s^\prime)}$ and $\{\widetilde{R}^k_h(s,a)\}_{(k,h,s,a)}$ also satisfies $\epsilon$-DP. Then it holds that the release of all $\pi_k$ is $\epsilon$-DP according to post-processing. Finally, the guarantee of $\epsilon$-JDP results from Billboard Lemma (Lemma 9 of \citet{hsu2014private}).

For utility analysis, because of Theorem 3.6 of \citet{chan2011private}, our choice $\epsilon^\prime=\frac{\epsilon}{3H\log K}$ in Binary Mechanism and a union bound, with probability $1-\frac{\beta}{3}$, for all $(k,h,s,a,s^\prime)$,
\begin{equation}
\begin{split}
|\widehat{N}^k_h(s,a,s^\prime)-N^k_h(s,a,s^\prime)|\leq O(\frac{H}{\epsilon}\log(HSAT/\beta)^2),\;\;
|\widehat{N}^k_h(s,a)-N^k_h(s,a)|\leq O(\frac{H}{\epsilon}\log(HSAT/\beta)^2),\\
|\widetilde{R}^k_h(s,a)-R^k_h(s,a)|\leq O(\frac{H}{\epsilon}\log(HSAT/\beta)^2).
\end{split}
\end{equation}
Together with Lemma \ref{lem:middle}, the Central Privatizer satisfies Assumption \ref{assump} with $\e=\widetilde{O}(\frac{H}{\epsilon})$.
\end{proof}

\begin{proof}[Proof of Theorem \ref{thm:central}]
The proof directly results from plugging  $\e=\widetilde{O}(\frac{H}{\epsilon})$ into Theorem \ref{thm:main}.
\end{proof}

\begin{proof}[Proof of Lemma \ref{lem:middle}]
For clarity, we denote the solution of \eqref{eqn:final_choice} by $\bar{N}^k_h$ and therefore $\widetilde{N}^k_h(s,a,s^\prime)=\bar{N}^k_h(s,a,s^\prime)+\frac{\e}{2S}$, $\widetilde{N}^k_h(s,a)=\bar{N}^k_h(s,a)+\frac{\e}{2}$.

When the condition (two inequalities) in Lemma \ref{lem:middle} holds, the original counts $\{N^k_h(s,a,s^\prime)\}_{s^{\prime}\in\mathcal{S}}$ is a feasible solution to the optimization problem, which means that
$$\max_{s^{\prime}}|\bar{N}^k_h(s,a,s^{\prime})-\widehat{N}^k_h(s,a,s^{\prime})|\leq \max_{s^{\prime}}|N^k_h(s,a,s^{\prime})-\widehat{N}^k_h(s,a,s^{\prime})|\leq \frac{\e}{4}.$$
Combining with the condition in Lemma \ref{lem:middle} with respect to $\widehat{N}^k_h(s,a,s^\prime)$, it holds that
$$|\bar{N}^k_h(s,a,s^{\prime})-N^k_h(s,a,s^{\prime})|\leq |\bar{N}^k_h(s,a,s^{\prime})-\widehat{N}^k_h(s,a,s^{\prime})|+|\widehat{N}^k_h(s,a,s^{\prime})-N^k_h(s,a,s^{\prime})|\leq \frac{\e}{2}.$$
Since $\widetilde{N}^k_h(s,a,s^\prime)=\bar{N}^k_h(s,a,s^\prime)+\frac{\e}{2S}$ and $\bar{N}^k_h(s,a,s^\prime)\geq 0$, we have\\
\begin{equation}\label{eqn:1}
\widetilde{N}^k_h(s,a,s^\prime)>0,\;\;\; |\widetilde{N}^k_h(s,a,s^{\prime})-N^k_h(s,a,s^{\prime})|\leq\e.
\end{equation}

For $\bar{N}^k_h(s,a)$, according to the constraints in the optimization problem \eqref{eqn:final_choice}, it holds that 
$$|\bar{N}^k_h(s,a)-\widehat{N}^k_h(s,a)|\leq\frac{\e}{4}.$$
Combining with the condition in Lemma \ref{lem:middle} with respect to $\widehat{N}^k_h(s,a)$, it holds that
$$|\bar{N}^k_h(s,a)-N^k_h(s,a)|\leq |\bar{N}^k_h(s,a)-\widehat{N}^k_h(s,a)|+|\widehat{N}^k_h(s,a)-N^k_h(s,a)|\leq \frac{\e}{2}.$$
Since $\widetilde{N}^k_h(s,a)=\bar{N}^k_h(s,a)+\frac{\e}{2}$, we have 
\begin{equation}\label{eqn:2}
N^k_h(s,a)\leq\widetilde{N}^k_h(s,a)\leq N^k_h(s,a)+\e.
\end{equation}

According to the last line of the optimization problem \eqref{eqn:final_choice}, we have $\bar{N}^k_h(s,a)=\sum_{s^\prime\in\mathcal{S}}\bar{N}^k_h(s,a,s^\prime)$ and therefore,
\begin{equation}\label{eqn:3}
    \widetilde{N}^k_h(s,a)=\sum_{s^\prime\in\mathcal{S}}\widetilde{N}^k_h(s,a,s^\prime).
\end{equation}

The proof is complete by combining \eqref{eqn:1}, \eqref{eqn:2} and \eqref{eqn:3}.
\end{proof}

\begin{proof}[Proof of Lemma \ref{lem:local}]
The privacy guarantee directly results from properties of Laplace Mechanism and composition of DP \citep{dwork2014algorithmic}.

For utility analysis, because of Corollary 12.4 of \citet{dwork2014algorithmic} and a union bound, with probability $1-\frac{\beta}{3}$, for all $(k,h,s,a,s^\prime)$,
\begin{equation}
\begin{split}
|\widehat{N}^k_h(s,a,s^\prime)-N^k_h(s,a,s^\prime)|\leq O(\frac{H}{\epsilon}\sqrt{K\log(HSAT/\beta)}),\;\;
|\widehat{N}^k_h(s,a)-N^k_h(s,a)|\leq O(\frac{H}{\epsilon}\sqrt{K\log(HSAT/\beta)}),\\
|\widetilde{R}^k_h(s,a)-R^k_h(s,a)|\leq O(\frac{H}{\epsilon}\sqrt{K\log(HSAT/\beta)}).
\end{split}
\end{equation}
Together with Lemma \ref{lem:middle}, the Local Privatizer satisfies Assumption \ref{assump} with $\e=\widetilde{O}(\frac{H}{\epsilon}\sqrt{K})$.
\end{proof}

\begin{proof}[Proof of Theorem \ref{thm:local}]
The proof directly results from plugging  $\e=\widetilde{O}(\frac{H}{\epsilon}\sqrt{K})$ into Theorem \ref{thm:main}.
\end{proof}

\end{document}